%% file: arxiv.tex
\documentclass{article} %
\usepackage[final]{colm2026_conference}

\usepackage{microtype}
\usepackage{hyperref}
\usepackage{url}
\usepackage{booktabs}
\usepackage{multirow}

\usepackage{lineno}

\usepackage{xcolor}
\usepackage{tcolorbox}
\usepackage{listings}
\usepackage{wrapfig}
\usepackage{array}
\definecolor{bg_prompt}{RGB}{245, 245, 245} 
\definecolor{bg_trace}{RGB}{235, 245, 255}  
\definecolor{darkblue}{RGB}{0, 0, 139}
\newtcolorbox{promptbox}[1]{
  colback=bg_prompt,
  colframe=gray!50,
  title=\textbf{#1},
  coltitle=black,
  fonttitle=\bfseries,
  boxrule=0.5pt,
  sharp corners,
  left=2mm, right=2mm, top=2mm, bottom=2mm
}

\definecolor{boxbg}{RGB}{248,249,250}
\definecolor{boxframe}{RGB}{180,185,190}
\definecolor{commentgray}{RGB}{90,90,90}

\newcommand\CRASP{\textsc{C-RASP}}
\newcommand\CRASPpl{\CRASP[\text{periodic},\text{local}]}
\newcommand\CRASPplshort{\CRASP[\text{Pos}]}

\newcommand{\cttn}[2]{\ensuremath{\textsc{\textbf{\#}} \left[ #1 \right] \; #2}}

\newcommand{\cif}[3]{\ensuremath{#1\;\mathbf{?}\; #2\; \textbf{:} \;#3}}

\input{math_commands}
\newcommand{\TCzero}{\ensuremath{\textbf{TC}^0}}
\newcommand{\ACzero}{\ensuremath{\textbf{AC}^0}}
\newcommand{\infrasp}{\ensuremath{\textsc{C*-RASP}}}

\newcommand{\Sem}[2]{[\![#1 ]\!]_{#2}}

\usepackage[capitalize,noabbrev]{cleveref}

\usepackage{amsmath}
\usepackage{amssymb}
\usepackage{mathtools}
\usepackage{amsthm}
\theoremstyle{plain}
\newtheorem{theorem}{Theorem}[section]

\newtheorem{lemma}[theorem]{Lemma}
\newtheorem{corollary}[theorem]{Corollary}
\theoremstyle{definition}
\newtheorem{definition}[theorem]{Definition}

\theoremstyle{remark}

\definecolor{darkblue}{rgb}{0, 0, 0.5}
\hypersetup{colorlinks=true, citecolor=darkblue, linkcolor=darkblue, urlcolor=darkblue}

\title{Barriers to Universal Reasoning With Transformers\\(And How to Overcome Them)}

\author{
    Oliver Kraus\thanks{Equal contribution. $\dagger$Co senior authors. Correspondence: \{okraus, ysarrof\}@lst.uni-saarland.de. \\ Link to Code: \url{https://github.com/coli-saar/BarriersToUniversalReasoningWTransformers}} \textsuperscript{ 1}, 
    Yash Sarrof\textsuperscript{$\ast$1}, 
    Yuekun Yao\textsuperscript{2}, 
    Alexander Koller\textsuperscript{$\dagger$1}, 
    Michael Hahn\textsuperscript{$\dagger$1} \\
    \vspace{0.2cm} 
    \textsuperscript{1}Saarland University \qquad \textsuperscript{2}Ohio State University
}

\begin{document}

\ifcolmsubmission
\linenumbers
\fi

\maketitle

\begin{abstract}
Chain-of-Thought (CoT) has been shown to empirically improve Transformers' performance, and theoretically increase their expressivity to Turing completeness. 
However, whether Transformers can learn to generalize to CoT traces longer than those seen during training is understudied. We use recent theoretical frameworks for Transformer length generalization and find that -- under standard positional encodings and a finite alphabet -- Transformers with CoT cannot solve problems beyond $\TCzero$, i.e. the expressivity benefits do not hold under the stricter requirement of length-generalizable learnability. However, if we allow the vocabulary to grow with problem size, we attain a length-generalizable simulation of Turing machines where the CoT trace length is linear in the simulated runtime up to a constant. Our construction overcomes two core obstacles to reliable length generalization: repeated copying and last-occurrence retrieval. We assign each tape position a unique signpost token, and log only value changes to enable recovery of the current tape symbol through counts circumventing both barriers. Further, we empirically show that the use of such signpost tokens and value change encodings provide actionable guidance %
to improve length generalization on hard problems. 
\end{abstract}

\section{Introduction}
\label{sec:introduction}

Chain of Thought (CoT) empirically improves the performance of Transformers \citep{vaswani2017attention} on reasoning tasks \citep{wei2022chain, nye2022show, kojima2022large, goyal2024think}.
This apparent increase in reasoning capability has been explained by theoretical findings showing that any Turing Machine (TM) can be simulated
by a Transformer  with CoT  \citep{perez2019turing, bhattamishra2020computational, wei2022chain, merrill2024expressive, qiu2024ask}, whereas Transfomers without CoT are restricted to the  complexity class $\TCzero$ \citep{merrill2023parallelism}. However these results point only to the \emph{existence} of such a simulation, but not to whether such a simulation can be \emph{learned} from finite CoT traces.  Indeed, a mounting body of empirical evidence shows that the accuracy of Transformers degrades sharply with CoT lengths --
across tasks including algorithmic \citep{anil2022exploring, he2026can}, logical \citep{kazemnejad2023positional, saparov2023language, saparov2023testing, saparov2025transformers}, planning \citep{stechly2024chain}, and symbolic reasoning \citep{mirzadeh2025gsm_symbolic, ehop-2025}.
This behavior is poorly understood from a theory perspective. 

Here, we study this question using recent advances in the theory of Transformer length generalization, specifically  $\CRASP$ \citep{yang2024counting, huangFormalFrameworkUnderstanding2025} and $\infrasp$ \citep{sarrof2026planning} and make two complementary theoretical contributions. In our \textbf{first key result}, we prove that 
under standard positional encodings and a finite alphabet, Transformers even with CoT 
will likely length-generalize only on problems within $\TCzero$, even when they
are trained to perform CoT reasoning. We show this by proving that any task (input and CoT trace) expressible in $\CRASP$, is also definable in $\TCzero$, strengthening the result of \cite{jiang2026softmax} which showed that $\CRASP$ CoTs are not Turing complete. 
This is a surprising discovery, as prior work \citep{feng2023towards, merrill2024expressive,li2024chain} attributes the success of CoT to transcending the $\TCzero$ barrier. While this is true for expressivity, we find it to be not true for the stricter requirement of length generalization. 

Two barriers impede length-generalization  of Turing-complete reasoning with Transformers: 
copying arbitrary strings, and
position-based retrieval (cf.\ Figure~\ref{fig:barrier-hero-fig}), both  challenging for length generalization \citep[e.g.][]{liu2024exposing,  zhou2024algorithms, ebrahimi2024your, huangFormalFrameworkUnderstanding2025, jobanputra2025born}. We address copying limitations using
 \emph{signpost tokens}, and circumvent retrieval limitations by a novel CoT scheme
encoding the values of a tape cell \emph{before} and \emph{after} a computation step
changing it. %
This construction allows us to encode the individual CoT steps in $\infrasp$
(extension of \CRASP\ to infinite alphabets). \cite{sarrof2026planning} shows that, under a idealized learned model and with an infinite alphabet, Transformers \emph{provably} length generalize on problems describable in $\infrasp$. 

Thus, in our \textbf{second key result} we prove that Transformers with CoT \emph{can} length-generalizingly learn to simulate any TM if they (a) get to use an infinite alphabet of \emph{signpost} symbols
and (b) the computations of the TM are encoded in the CoT in the right way.

We also attempt to empirically validate our theoretical predictions.
We find that Transformers trained from scratch do not length-generalize in CoT reasoning on Parity, Boolean Evaluation, and Permutation, but a signpost symbol inventory growing with instance size and a value-change format can boost length-generalization. These insights help get better prompting performance on our tasks on pretrained LLMs of different sizes 
(Qwen2.5-7B-Instruct, Mistral-Small-24B-Instruct-2501, Llama-3.3-70B-Instruct). 
Taken together, our results may help explain why existing methods for CoT
reasoning with Transformers have empirically exhibited limitations with respect
to length generalization. Turing completeness is not guaranteed, but rather depends heavily on the specific formatting of the CoT. By taking the learnability of the CoT format into account, we can create more effective schemes that may actually realize this underlying computational power.

\begin{figure}
    \centering

\includegraphics[width=\textwidth]{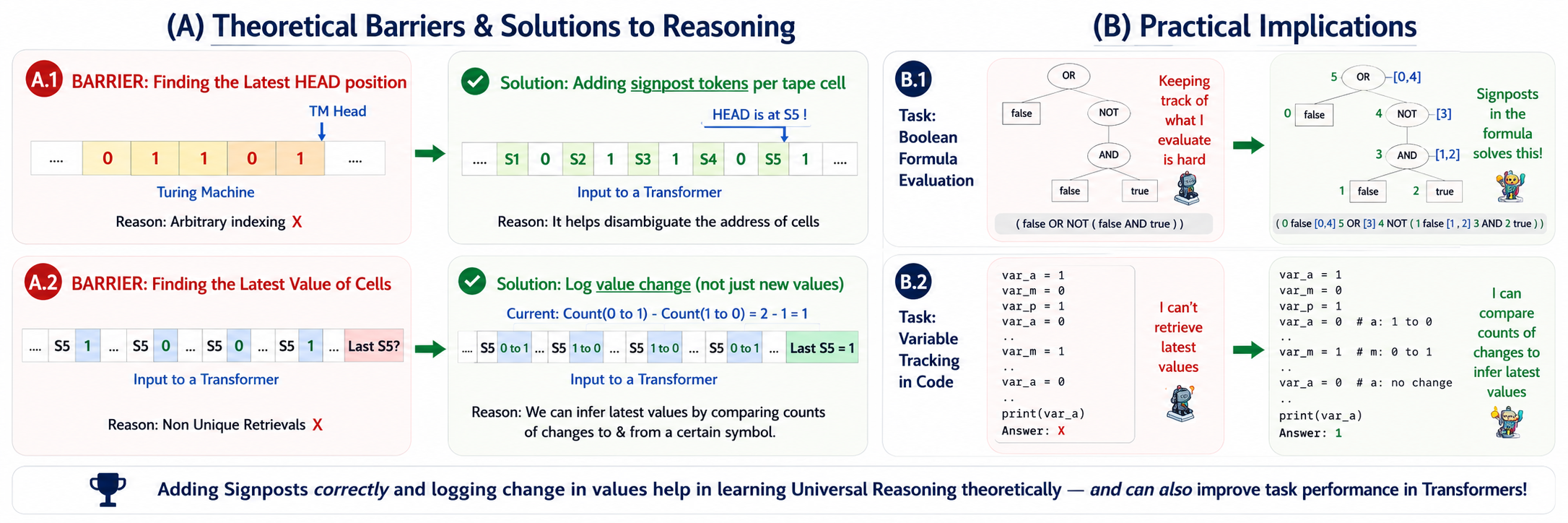}

    \caption{(A.1) 
    Simulating TM computations requires arbitrary addressing and retrieval of previously written tape positions.
    Simulating this does not length-generalize; we address it by assigning unique \emph{signpost tokens} to cells. 
    (A.2) Even with signposts, recovering the latest content of a cell from the CoT trace becomes a non-unique retrieval problem \citep{liu2024exposing}, for which also there is no length generalizable solution. Hence, we use a \emph{value change encoding} to explicitly record changes rather than just values (i.e., record the new value relative to the old value), but only when there is a change. This allows computing the current value by comparing the counts of changes between each symbol pair (Section~\ref{subsec:turingcomplete}). These theoretical ideas have implications for formatting inputs and CoTs alike. (B.1) Putting signposts tokens according to the post-order traversal order of a Boolean formula in the input can help enable a transformer solve this hard task (Section~\ref{subsec:scratch_experiments}). (B.2) Ascertaining the final value of a variable after multiple intermediate assignments when \emph{the changes of values are recorded} in comments can make LLMs output the correct answer (Section~\ref{subsec:prompting_main}).}
    \label{fig:barrier-hero-fig}
\end{figure}

\section{Background}

\subsection{Related Work}

\textbf{Empirical failures seen for CoT \& steps to improve them.} A substantial literature shows that CoT does not, by itself, guarantee length generalization. Across algorithmic \citep{anil2022exploring, he2026can}, logical \citep{kazemnejad2023positional, saparov2023language, saparov2023testing, saparov2025transformers}, planning \citep{stechly2024chain}, and symbolic reasoning tasks \citep{mirzadeh2025gsm_symbolic, ehop-2025}, prior work finds that performance often degrades sharply with increasing problem length or number of reasoning steps, even when intermediate reasoning traces are provided or elicited,  suggesting that the benefits of CoT are fragile in out-of-distribution settings.  At the same time, several works show that the \emph{form} of CoT can matter substantially.  Adding padding \citep{Jelassi2023Length}, modified scratchpad formats \citep{abbe2024how}, and more structured intermediate representations or specific positional encodings (PEs) \citep{shen2023positional, mcleish2024transformers, JelassiBKM24} have been shown to improve length generalization across tasks, and even on harder tasks such as parity \citep{kim2024transformers} which are hard to do without CoT \citep{hahn-rofin-2024-sensitive}. However, these gains are highly sensitive to the chosen representation and do not imply that arbitrary CoT formats length-generalize. Our construction with signpost tokens in Turing Machines is heavily inspired by, and generalizes, the idea of index hints \citep{nogueira2021investigating, bueno-etal-2022-induced, zhou2024algorithms, ebrahimi2024your} being used across the input and CoT to help in generalization, and provides a theoretical justification of why index hints help. 

\textbf{Prior Turing Completeness results for Transformers.} Most prior Turing-completeness constructions for Transformers with CoT largely concern \emph{expressivity} 
\citep{perez2019turing, bhattamishra2020computational, wei2022statistically, qiu2024ask, malach2024autoregressive, amiri2025lowerboundschainofthoughtreasoning, merrill2024expressive}. Expressivity implies that in principle a transformer could be hard coded to simulate a TM computation, but it does not address the question of whether such simulations can be learnt with any learning model. There also exist prior constructions that address learnability of CoT, we contrast these results with ours in Section~\ref{subsec:prior_turing}.

\textbf{Link to Circuit Complexity and existing Barriers.} The abilities of Transformers are often studied via circuit complexity \citep{vollmer1999introduction}, which formalize the difficulty of problems for highly parallelized architectures \citep{merrill2026why}. 
Prior work bounds the power of standard Transformers (no CoT) in terms of the constant-depth class \(\TCzero\) and its subclass \(\ACzero\) \citep{hahn2020theoretical, hao2022formal, merrill-etal-2022-saturated} while with CoT they can achieve Turing completeness \citep{perezJMLR, bhattamishra2020computational,merrill2024expressive}. 
$PARITY$ (is the number of $1$s even in a binary string ?) lies in \(\TCzero\). 
On the other hand,  \emph{Boolean formula evaluation} (evaluating the truth value of a propositional logic formula without variables) and  \(S_5\) permutation-product problems (tracking order of objects after a series of permutations) are not in \(\ACzero\), believed to be outside even of \(\TCzero\), and thus -- under standard conjectures -- believed to be impossible to express for Transformers without CoT \citep{merrill2023parallelism}.
It has thus been argued that CoT, by achieving Turing completeness, allows Transformers to transcend $\TCzero$ \citep[\emph{inter alia},][]{li2024chain}.

\subsection{Theoretical Frameworks for Length Generalization: {\CRASP} \& {\infrasp}} \label{subsec:bg_crasp_infrasp}

\textbf{C-RASP.}
Counting-RASP introduced by \citet{yang2024counting} has become the standard formalism to study the abilties of decoder-only softmax transformers. 
Recently \cite{huangFormalFrameworkUnderstanding2025} showed that membership in {\CRASP} and its variant with some positional primitives {\CRASPplshort}\footnote{shorthand used in \cite{jobanputra2025born} to refer to {\CRASPpl}\citep{huangFormalFrameworkUnderstanding2025}}
could be used to give length generalization guarantees for Transformers with No Positional Encodings (NoPE) and Absolute Positional Encodings (APE) respectively, under an idealized learning model. Since, {\CRASP}
has become a central tool for reasoning about length generalization in a theoretically principled manner \citep{huangFormalFrameworkUnderstanding2025, chen2025nonasymptotic, yang2025kneedeep}. Further, empirically, it has also been highly predictive: tasks expressible in {\CRASP} tend to generalize with NoPE and those expressible in {\CRASPplshort} generalize with APE transformers, while tasks shown to lie outside tend to fail when tested on longer lengths \citep{huangFormalFrameworkUnderstanding2025, jobanputra2025born}.

\textbf{{\infrasp}.} {\CRASP} assumes a fixed finite alphabet. 
{\infrasp} \citep{sarrof2026planning}, an extension of {\CRASP} uses a split alphabet \(\Sigma \cup \mathcal C\), where \(\Sigma\) is finite and \(\mathcal C\) is an unbounded alphabet. Its key new primitive is  a local \emph{match predicate}, which allowed the program to compare symbols from \(\mathcal C\) relationally without memorizing their absolute identities. We will here use $\mathcal C$ as an alphabet of signpost tokens. Like {\CRASP}, membership in {\infrasp} also implies a length generalization guarantee for NoPE and APE transformers (depending on usage of positional primitives) under an idealized learning model \citep{sarrof2026planning}.

\textbf{Learning Model used for Generalization guarantees.} \label{subsec:bg_learning_model} Both \(\CRASP\) and \(C^*\)-RASP are used in conjunction with idealized learning models that formalize the notion of length generalization as the \emph{asymptotic convergence} of a sequence of transformers -- trained on increasing lengths and/or on an increasing vocabulary -- to the same underlying
algorithm, enabling generalization. 
The guiding intuition is that a target task generalizes when it requires only local operations for success, requirements formalized as translation invariance (interactions between parameters should depend only
on relative differences between positions and symbols) and locality (beyond a threshold such interactions should vanish).
A learner gets trained on instances of bounded size and/or bounded alphabet. 
Learning is the process of selecting from a restricted hypothesis class of finite precision transformers that satisfy these constraints, a model that fits the given training data perfectly, while respecting a regularizer that favors simplicity in models. $\CRASP, \infrasp$ serve as a more human readable interface to these ``algorithms''. The main result for both states that the inference procedure picks out a length generalizable solution if the task can be represented by $\CRASP$ or $\infrasp$. The key difference \footnote{This results in many changes in the technical specification of the learning model, from that in \citet{huangFormalFrameworkUnderstanding2025}, but this does not affect our results, as we operate at the RASP abstraction level.} is that while \citet{huangFormalFrameworkUnderstanding2025} put these constraints only on positions, \citet{sarrof2026planning} extended them to input symbols.

\section{Theoretical Results}

\subsection{Formalizing Length-Generalizable Learning of CoT}
\paragraph{Formalization of CoT generation.}

In {\CRASP}, one evaluates the program at the final input token to check if the  input string is accepted. However, for  CoT, beyond acceptance of a string, we also need a way to generate intermediate tokens autoregressively with {\CRASP}. We therefore introduce a corresponding autoregressive formalization of CoT applying both to Transformers and to CoT {\CRASP} programs,  closely aligning with \cite{jiang2026softmax}.

Let $\Sigma$ be the input alphabet, and let $\Xi \supseteq \Sigma$ be a finite alphabet of symbols that may appear during chain-of-thought generation. We assume that $\Xi$ contains two distinguished symbols
$\langle SEP \rangle,\ \langle EOS \rangle \in \Xi \setminus \Sigma$
used to mark the final answer and the end of generation, respectively. At the level relevant here, both transformers and CoT C-RASP programs induce a deterministic \emph{next-token map} $M : \Xi^* \to \Xi$
which specifies, for each current prefix, which token is generated next. Informally, we start with input prefix $w$, and the model emits a token at a time, conditioning each prediction on the input and the trace generated so far.

\begin{definition}[Autoregressive CoT generation]\label{def:cot-generation}
Let $M : \Xi^* \to \Xi$ be a next-token map, and let $w \in \Sigma^*$.
We say that $M$ \emph{generates} a string
$U = U_1 \cdots U_m \in \Xi^*$
on input $w$ if for each $k=1,\dots,m$,  $M(wU_1\cdots U_{k-1}) = U_k$
\end{definition}

\textbf{Transformers with CoT.}
We model a decoder-only transformer as a length-preserving map $T : \Xi^* \to \Xi^*$
where, for any string $x = x_1 \cdots x_m \in \Xi^*$, the symbol $T(x)_m$ is interpreted as the token predicted after reading the prefix $x_1 \cdots x_m$. With a slight abuse of notation, we write $T(x) := T(x)_{|x|}$ for this next-token prediction. Thus every decoder-only transformer induces a next-token map $x \mapsto T(x)$.

\textbf{C-RASP with CoT.}
A \emph{CoT C-RASP program} over $(\Sigma,\Xi)$ consists of a standard C-RASP program together with a family of definitions for every $\sigma \in \Xi$,  $\operatorname{OUTPUT}(\sigma) \defeq \varphi_{\sigma}$
where each $\varphi_{\sigma}$ is a Boolean C-RASP expression, subject to the requirement that for every input prefix $w \in \Xi^*$, exactly one of the formulas $\varphi_{\sigma}$ is true at the last position.%
\footnote{Compared to \citet{jiang2026softmax}, we write $\operatorname{OUTPUT}(\sigma)$ instead of $O_\sigma$, and require uniqueness rather than relying on an ordering of output clauses. This simplification leaves expressiveness unchanged.}
For a Boolean C-RASP expression $\varphi$ and a string $w$, we write $\Sem{\varphi}{w}(i) \in \{0,1\}$ for the truth value of $\varphi$ at position $i$ when evaluated on $w$. This yields a next-token map $S : \Xi^* \to \Xi$
defined by $S(w)=\sigma
\text{  iff  } 
\Sem{\varphi_\sigma}{w}(|w|)=1.$

\begin{figure}[t]
    \centering
    \includegraphics[width=0.9\linewidth, keepaspectratio=false, height=2cm, trim=0 0 0 1.1cm, clip]{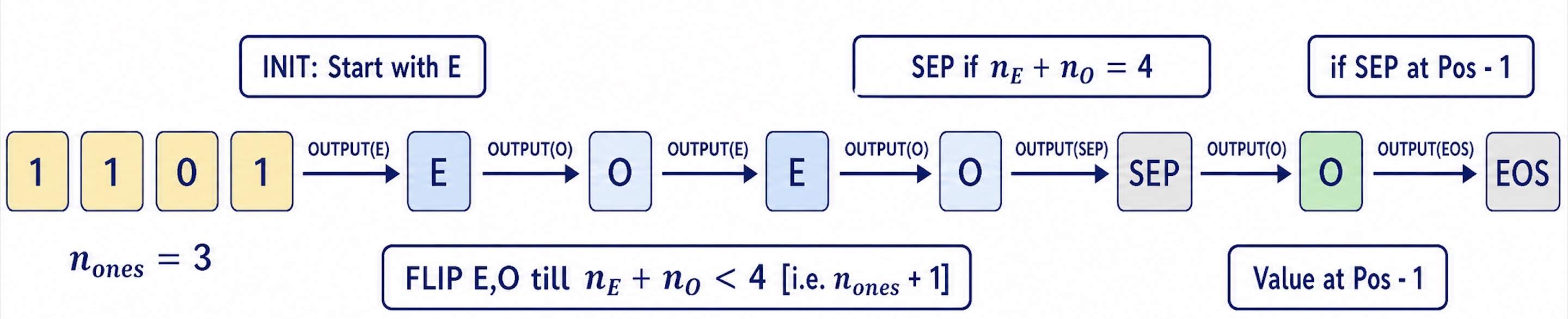}
    \caption{\textbf{Autoregressive CoT Trace for \textsc{Parity}.}
    For the input \texttt{1101}, which contains $3$ ones, the CoT C-RASP program generates the trace \texttt{E O E O}, always starting with emitting an even \texttt{E} token and flipping till the count of even and odd tokens produced crosses the number of ones (here, \(n_{\mathrm{ones}}+1=4\)). It then emits \(\langle SEP\rangle\), copies the final parity token \texttt{O} using the local {\CRASPplshort} predicate as the answer, and terminates with \(\langle EOS\rangle\). Note that only one \(\operatorname{OUTPUT}(\cdot)\) clause is active at each autoregressive step, which therefore leads to a deterministic map. The fully formal {\CRASP} program is given in Appendix~\ref{sec:crasp_cot_example}.}
    \label{fig:parity-cot}
\end{figure}

\begin{definition}[Computation of a partial function by CoT]\label{def:cot-computes}
Let  $F : \Sigma^* \to \Sigma^* \cup \{\bot\}$.
Let $M : \Xi^* \to \Xi$ be a next-token map. We say that $M$ \emph{computes} $F$ with CoT if, for every input $w \in \Sigma^*$, either $F(w) \in \Sigma^*$, and $M$ generates a string ending in $\langle SEP \rangle\, F(w)\, \langle EOS \rangle$  or  $F(w)=\bot$, and no string generated by $M$ on input $w$ contains $\langle EOS \rangle$.
\end{definition}

In case (1), the generated prefix before $\langle SEP \rangle$ is the chain of thought, while the substring between $\langle SEP \rangle$ and $\langle EOS \rangle$ is the final answer. In case (2), the absence of $\langle EOS \rangle$ represents non-termination. Note that the use of partial functions is important for modeling Turing machines, which may not terminate on a given input. Here, the second condition corresponds to non-termination of the computation, equivalent to the function $F$ being undefined.

\textbf{Length-generalizable learning of CoT.} We use a function \(f_{\mathrm{TestLength}}\) to specify the train-to-test scaling regime: after training on prefixes of length at most \(n\), the learned model is evaluated on prefixes of length up to \(f_{\mathrm{TestLength}}(n)\). 
We now adapt the idealized learning model of \citet{huangFormalFrameworkUnderstanding2025, jiang2026softmax} from language recognition to autoregressive CoT generation for functions. Informally, once the learner has seen all bounded CoT prefixes up to length \(n\), it should eventually recover the same underlying autoregressive computation and therefore continue it correctly out to the larger test scale \(f_{\mathrm{TestLength}}(n)\). 
Following \cite{huangFormalFrameworkUnderstanding2025}, the learning model assumes transformers with NoPE or APE.
\begin{definition}[Length-generalizable learning of CoT over a finite alphabet]
\label{def:cot-learnable-finite}
Let \(F : \Sigma^* \to \Sigma^* \cup \{\bot\}\) be a partial function, and let $f_{\mathrm{TestLength}}:\mathbb N\to\mathbb N$.
We say that \(F\) is \emph{length-generalizably learnable with CoT} if there exists a \(M:\Xi^*\to\Xi\) computing \(F\) with CoT such that the following holds. For each \(n=1,2,3,\dots\), use the idealized learning procedure of \citet{huangFormalFrameworkUnderstanding2025} to choose a Transformer \(T_n\) that for every prefix $wU_1\cdots U_{j-1}$ of a CoT generated by $M$ on input $w \in \Sigma^*$ with $|wU_1\cdots U_{j-1}| \le n$, predicts the correct next token $U_j$. Then there exists \(N_0\) such that for all \(n>N_0\), the next-token map induced by \(T_n\) agrees with \(M\) on every such prefix of length at most \(f_{\mathrm{TestLength}}(n)\).
\end{definition}

\subsection{Limited length generalization under a finite alphabet} \label{subsec:tc0result}

Prior work \citep{feng2023towards,merrill2024expressive, li2024chain} suggests that CoT helps because it allows Transformers to transcend $\TCzero$. 
Our first result shows that this benefit is not expected to hold under the requirement of guaranteed length generalization.
We state the theorem in the special case of deciding membership in a language:

\begin{theorem}\label{thm:finite-alphabet-barrier}
Let $\mathcal{L} \subset \Sigma^*$, and let $F_\mathcal{L} : \Sigma^* \rightarrow \{1, 0\}$ be its indicator function (i.e., $F_\mathcal{L}(w) = 1 \Leftrightarrow w \in \mathcal{L}$). If $F_\mathcal{L}$ has a CoT in $\CRASPplshort$, then $\mathcal{L}$ is definable in $\TCzero$.

\end{theorem}
For this, we prove that the output of any CoT trace expressible in $\CRASPplshort$ is also definable in $\TCzero$. Thus, unlike in the case of expressivity, we do not go beyond $\TCzero$ under the stronger requirement of guaranteed length generalization over a fixed finite alphabet. The theorem is stated for decision problems; an analogous statement holds for function problems with polynomially bounded output length. The proof is given in Appendix~\ref{app:theorem-1}. 

\begin{corollary}
    When $\mathcal{L}$ is not in $\TCzero$, then, for some $f_{TestLength} : \mathbb{N} \rightarrow \mathbb{N}$, under CoT learning notion of Definition~\ref{def:cot-learnable-finite} based on the learning model of \citet{huangFormalFrameworkUnderstanding2025}, length generalization from training length $n$ to test length $f_{TestLength}(n)$ will fail for infinitely many $n$.
\end{corollary}

While limit transformers and the learning model of \citet{huangFormalFrameworkUnderstanding2025} are defined for NoPE/APE transformers, \citet{yang2025kneedeep} show that RoPE/AliBi transformers can be expressed in $\CRASPplshort$, showing that an equivalent length generalization failure is also expected for those types of PEs.

\subsection{Turing completeness under an infinite alphabet} \label{subsec:turingcomplete}

\paragraph{\texorpdfstring{$\infrasp$}{C*-RASP} with CoT.}
We now extend the CoT formalism to the infinite-alphabet setting. The key difference from finite-alphabet CoT is that the generated trace may need to contain symbols from the unbounded alphabet $\mathcal C$, not only symbols from a fixed finite vocabulary.

\begin{figure}
    \centering

    \includegraphics[width=\textwidth, keepaspectratio=false, height=4.3cm]{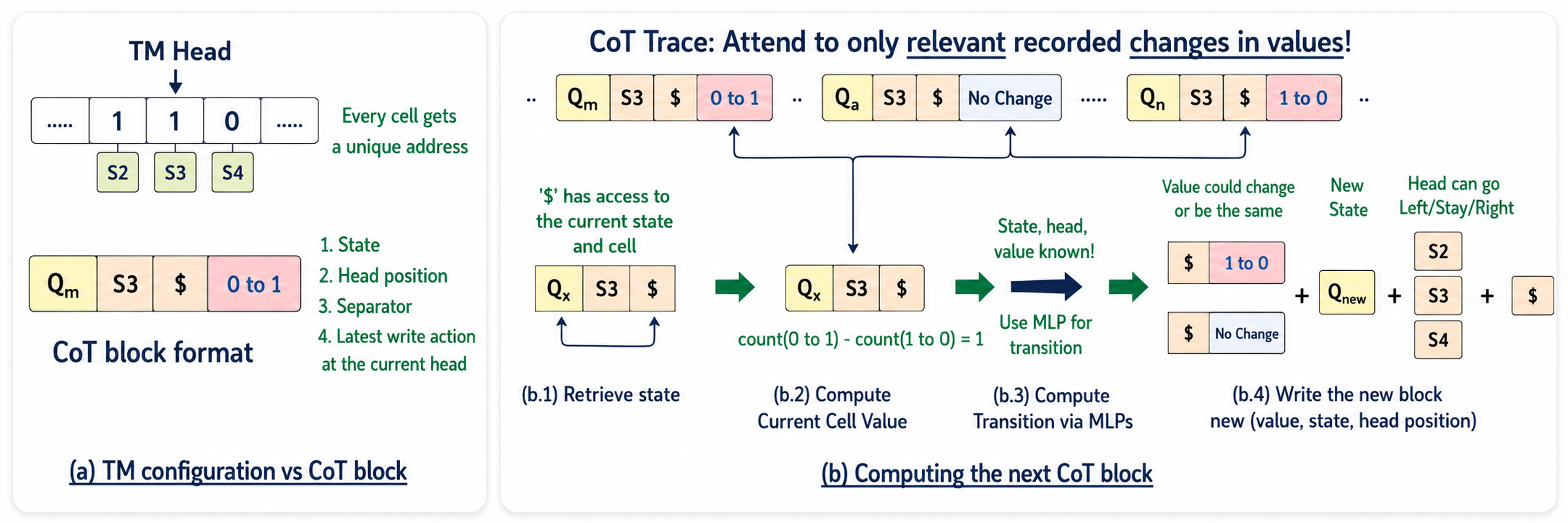}
    \caption{A TM simulation step for a single tape in our construction (same logic extends to a multi tape setup). (a) Each cell is assigned a signpost token indicating its address. We always output the CoT in blocks containing the state, cell where the head currently is, a separator, and then the change in value at the current cell after transitioning. (b) In each block, the separator has access to the current state and the current cell, as they are at fixed offsets from it. Furthermore, the separator can compute the current value on the current cell by attending to all logs of value changes done on the cell in the past CoT trace, and comparing the counts of how many times it was turned to specific values. Now, given the current state, head position, and the value under the head position, we can use an MLP to compute the next transition. We first log the write action (no change in value or a change in value). Then, we output the new state, the new head position based on whether the head moved left right or stayed, and a separator again. Full construction in Appendix~\ref{thm:positive_restate}.}
    \label{fig:tm_visualize}
\end{figure}

Let $\Sigma$ be a finite alphabet, and $\Xi \supseteq \Sigma \cup \{\langle SEP \rangle,\langle EOS \rangle\}$ be the same CoT finite alphabet of symbols used before. In addition we have $\mathcal C \cong \mathbb{N}$ which acts as an unbounded alphabet\footnote{We want $\mathcal C$ to have fixed ordering, choosing it to be isomorphic to $\mathbb{N}$ is simply a convenient choice.} that both the CoT and the input have access to. We assume that any input word presented as input $w = w_1\cdots w_n \in \Sigma^*$, is first annotated with signpost tokens from $\mathcal{C}$. 
$\mathsf{Annotate}(w) := w_1 c_1\, w_2 c_2 \cdots w_n c_n$
where $c_1,c_2,\dots \in \mathcal C$ are the first $n$ elements of $\mathcal C$. Intuitively, $\mathsf{Annotate}(w)$ interleaves each input token with its unique signpost token. These signposts will help us overcome the difficulty of arbitrary copying identified in prior work \citep{JelassiBKM24,zhou2024algorithms,huangFormalFrameworkUnderstanding2025} by allowing arbitrary indexing. 

Similar to its counterpart, a CoT $\infrasp$ program over $(\Sigma,\mathcal C,\Xi)$ consists of a standard $\infrasp$ program together with two output clauses. The first is $\operatorname{OUTPUT}(\sigma) \defeq \varphi_\sigma$ for $\sigma \in \Xi$. The second is $\operatorname{OUTPUT\_SIGNPOST}(k, d) \defeq \varsigma_{(k, d)}$ where  $k > 0$ is a single finite constant representing a fixed offset and a small constant $d = \{-1, 0, 1\}$, which represents a move direction (left, stay, right). Here $\varphi_\sigma, \varsigma_{(k, d)}$ are Boolean $\infrasp$ expressions and the outputs correspond to outputting a finite symbol $\sigma$ in the first case, and outputting a specific signpost token in the second case. Similar to the formalization of CoT with {\CRASP} we will require that for every prefix, exactly one output clause is active at the last position. Thus for $\mathsf{OUTPUT\_SIGNPOST_{(k,d)}}$, our constants $k$ will impose a local relation through $\varsigma$ so that a specific occurrence of a signpost is singled out, and then we will use our direction $d$ value to get either the same signpost, or the signpost just smaller than it, or the one just bigger that it \footnote{This is where the ordering of $\mathcal{C}$ matters.}. This induces a deterministic next-token map, with the codomain $\Xi \cup \mathcal C$.

\begin{definition}[Computation by CoT $\infrasp$]\label{def:infrasp-cot-computes}
Let $F : \Sigma^* \to \Sigma^* \cup \{\bot\}$. We say that a CoT $\infrasp$ program $S$ \emph{computes} $F$ if, for every $w \in \Sigma^*$, if either
$F(w) \in \Sigma^*$, and $S$ generates a string ending in $\langle SEP \rangle\ \mathsf{Annotate}(F(w))\, \langle EOS \rangle$ on input $\mathsf{Annotate}(w)$ holds or $F(w)=\bot$, and $S$ does not generate any string containing $\langle EOS \rangle$ on input $\mathsf{Annotate}(w)$ holds.
\end{definition}

\begin{definition}[Length-generalizable learning of CoT over an infinite alphabet]
\label{def:cot-learnable-infinite}
Let \(F : \Sigma^* \to \Sigma^* \cup \{\bot\}\) be a partial function, and let
$f_{\mathrm{TestLength}}:\mathbb N\to\mathbb N$.
We say that \(F\) is \emph{length-generalizably learnable with CoT} over \((\Sigma,\mathcal C,\Xi)\) if there exists a $M:(\Xi\cup\mathcal C)^*\to(\Xi\cup\mathcal C)$
computing \(F\) with CoT such that the following holds. 
For each \(n\in\mathbb N\), use the idealized learning procedure of \cite{sarrof2026planning} to choose a Transformer \(T_n\) that satisfies $T_n(\mathsf{Annotate}(w)U_1\cdots U_{j-1}) = U_j$ for every input $w \in \Sigma^*$ and every generated string $U_1 \cdots U_m$ of $M$ on $\mathsf{Annotate}(w)$, whenever $|\mathsf{Annotate}(w)U_1\cdots U_{j-1}| \le n$. Then there exists $N_0$ such that for all $n \ge N_0$, the next token map induced by $T_n$ agrees with $M$ on every such prefix of length at most $f_{\mathrm{TestLength}}(n)$.
\end{definition}

It should be noted that training instances can only ever have a finite subset of signpost tokens, but we learn to generalize on many more symbols by training them with offsets, as suggested in \cite{sarrof2026planning}. The idea is that in each annotation, instead of $c_1$, we choose a random offset in the range $(1, f_{\mathrm{TestLength}} -|w|)$ and annotate starting $c_{\textit{offset}}, \dots c_{\textit{offset} + |w|}$. This idea is analogous to how longer APE encodings are trained. We now show that CoTs can in fact universally length-generalize when allowed access to such an infinite alphabet:

\begin{theorem} \label{thm:positive}
\textit{($C^*$-RASP CoT is Turing-complete)}
Let $F$ be a semi-computable function computed by a Turing machine $\mathcal{T}$.
Then there is a CoT $C^*$-RASP program $S$ computing it.

\textit{($C^*$-RASP CoTs have length bounded by TM time complexity)}
Furthermore, there is a universal constant $C$ such that, when $\mathcal{T}$ terminates on $w$ in $N$ steps, then $S$ generates a string of length $C \cdot N_{tapes} \cdot (N + |w|)$ ending in $\langle EOS \rangle$ on $Annotate(w)$, where $N_{tapes}$ is the number of tapes of $\mathcal{T}$.
\end{theorem}

The technical key is a \emph{value-change} encoding of the TM run: the CoT records only those events in which a tape cell changes its value. This makes it possible to recover the current symbol at a given indexed position using the matching primitives of $\infrasp$, while keeping the trace length linear in the running time of the TM (cf.\ Figure~\ref{fig:tm_visualize}). Full proof in Appendix~\ref{thm:positive_restate}

\begin{corollary}[Length Generalization for CoT]
    \label{cor:cot-positive}
{Let $f_{TestLength} : \mathbb{N} \rightarrow \mathbb{N}$ be any function diverging to infinity, and let $F$ be a semi-computable function. Then \(F\) is length-generalizably learnable with CoT over \((\Sigma,\mathcal C,\Xi)\) in the sense of Definition~\ref{def:cot-learnable-infinite}.
}
\end{corollary}

\paragraph{CoT Formatting Matters.}
Theorem~\ref{thm:finite-alphabet-barrier} and~\ref{thm:positive} emphasize how important formatting the CoT aptly is. 
PARITY is in $\TCzero$, and thus may have a CoT even without the use of any infinite alphabet, however it still needs apt formatting that is in $\CRASP$ for generalization to succeed. In Section~\ref{subsec:scratch_experiments}, we show, how with apt value change formatting (see  Figure~\ref{fig:parity-cot}) Transformers perform well at longer lengths for PARITY, while a standard CoT (that doesn't respect the constraints of $\CRASP$) doesn't. The use of signpost tokens during the TM simulation suggests why using index hints in problem formulations leads to higher performance \citep{nogueira2021investigating, bueno-etal-2022-induced, Zhou2024Transformers, ebrahimi2024your}. 

\section{Empirical Validation}

\subsection{Models trained from scratch} 
\label{subsec:scratch_experiments}

\textbf{Setup.}
We confirm our theoretical results on three algorithmic reasoning tasks: Parity, Boolean Evaluation, and S5 Permutation \citep{kim-schuster-2023-entity, li2025how} (each described in Appendix \ref{appendix:task-descriptions}). For each, we train Transformers from scratch on CoT traces augmented with signpost tokens or a value-change format, and compare their length generalization performance against baseline models trained on CoT traces without these fixes. The CoT formatting aside, all models use the same architecture and training setup.

\textbf{Model.}
We train GPT-2 models with APE  \citep{radford2019language}, $6$ layers, $8$ attention heads and embedding dimension $512$, for a total of ~$25$M parameters. We use a custom whitespace tokenizer for each task which only contains the tokens that appear in the respective CoT traces. We train models on maximum sequence lengths of 30 and 50. Models are trained with a standard language modeling objective using cross entropy loss over the full CoT trace (not the prompt) and evaluated with greedy decoding. For generalization with APE, we follow \cite{huangFormalFrameworkUnderstanding2025} and use random offsets for positions in every training instance to learn the OOD PEs. We apply a similar offset trick to train the signpost tokens for OOD lengths as proposed in \cite{sarrof2026planning}. Thus, every sample $w$, instead of starting from position $1$, can start from a random offset $o_p$, where $1 \le o_p \le (\mathsf{max\_test\_len} - |w|)$. Similarly the signpost tokens instead of always starting from $c_1$, start from an offset $o_c$, where $1 \le o_c \le (\mathsf{max\_test\_len}/2 - |w|)$. \footnote{This  helps train the signpost tokens and PEs for longer lengths. For generalization to higher lengths, the upper bounds of $o_c, o_p$ can be aptly increased.} 

\textbf{Evaluation.} We tune hyperparameters separately for each task and CoT format and evaluate 7 random seeds across 3 hyperparameter configurations respectively
(Details in Appendix~\ref{appendix:hyperparams}). 
We report results based on the top 3 performing seeds (selecting the best-performing checkpoint on the validation set for each seed) on a held-out validation set spanning both in-distribution (ID) and out-of-distribution (OOD) lengths. Our ID validation dataset contains $1000$ samples with lengths in $[1 ; \mathsf{max\_train\_len}]$ and the test OOD set contains 500 samples with lengths in $[\mathsf{max\_train\_len}+1 ; \mathsf{max\_train\_len} * 2]$ per task. Evaluation on the test set is carried out with 500 samples per test length.

\textbf{Results.} Figure~\ref{fig:50_results_main} shows that signposts and value-change encodings improve length generalization. On PARITY, \emph{value-change} yields robust generalization to $2$x lengths and a naive CoT version fails. On S5 and Boolean Evaluation, adding \emph{signpost tokens} leads to marked  improvement in length generalization compared to the version without. 
On S5, the \emph{signpost} CoT generalizes to $1.7$x of the training length, while the combination of \emph{signposts} and \emph{value-change} notably improves performance further.
We also study a binary variant of S5 to better analyze the error cases with just signposts by comparing it to the case when value-change encodings are added. We find that the model struggles to correctly update states at longer test lengths with just \emph{signposts}, and \emph{value-change} encodings solves this issue, in line with the theoretical prediction that they help in overcoming the non-unique retrieval barrier.
Interested readers can find detailed analysis about this in Appendix~\ref{appendix:s5_discussion}, \ref{appendix:binary_permutation_experiments}. 
We next ask whether the same formatting principles help in designing better prompts for pretrained LLMs and improving performance without training from scratch.

\begin{figure}
    \centering
    \includegraphics[width=1.0\linewidth]{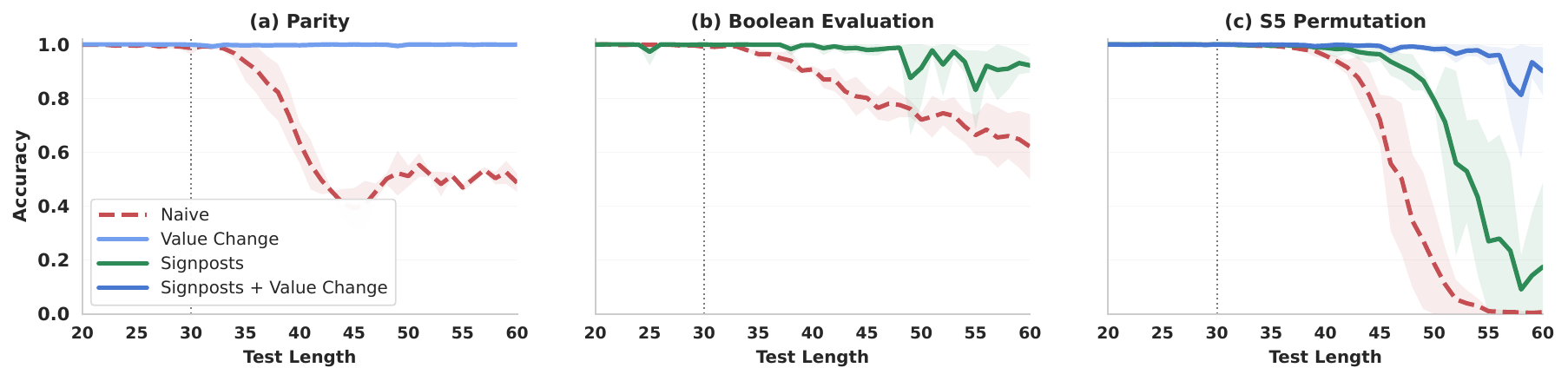}
    \caption{Length generalization from training length $30$ to test length $60$. We report results on PARITY, Boolean Evaluation and S5; (training-length 50 have similar curves, we provide them in the Appendix, Figure~\ref{fig:len_50_curve}. In all cases, the observed trends align with our theoretical predictions. Relative to the naive formulation, Signposts improve generalization on Boolean Evaluation, value-change improves generalization on PARITY, while the combined version of Signposts and value-change yields the strongest generalization on S5.}
    \label{fig:50_results_main}
\end{figure}

\subsection{Prompting LLMs}
\label{subsec:prompting_main}

\begin{figure}[t!]
    \includegraphics[width=\textwidth]{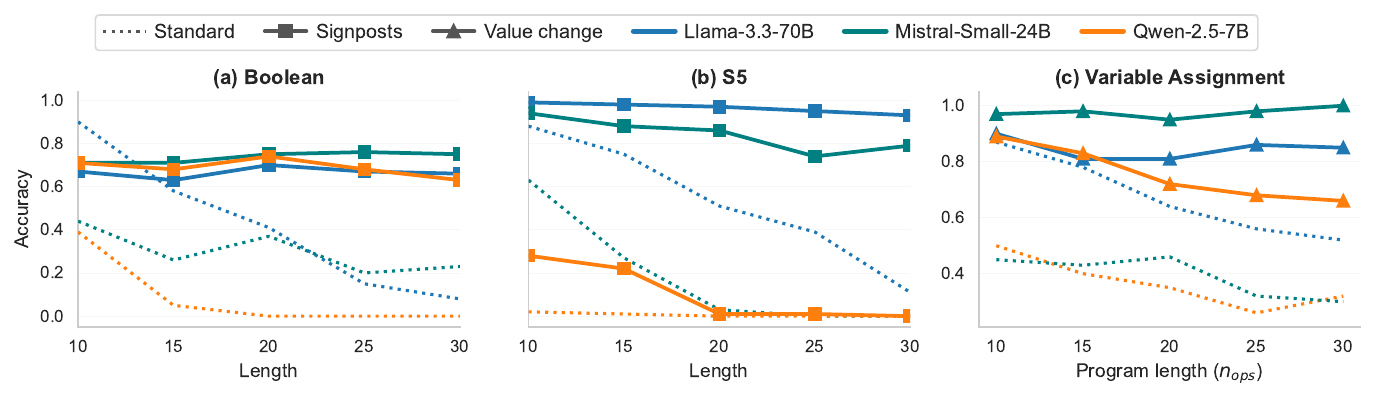}
\caption{We prompt open-source Transformer-based LLMs spanning different model families and scales to evaluate whether the theoretically principled suggestions of using signpost tokens and value-change logging can yield practical benefits in formatting inputs and CoT. In (a) and (b), we observe that instructing models to use signposts via few-shot examples improves their performance on Boolean evaluation and S5 tasks at longer sequence lengths; in contrast, performance without these tokens drops. In (c), we test the effectiveness of value-change logging by formatting the input of a code-execution task and asking \emph{only for the final output}. When the code is annotated with comments logging variable value changes, models outperform standard formatting. Our results suggest that theoretically motivated CoT and input designs can yield empirical gains out-of-the-box and also hint at promising directions for future pretraining mixtures. Prompt templates and details in Appendix \ref{appendix:prompting-experiments}.}
\label{fig:prompting-experiments}
\end{figure}

\textbf{Signposts:} Since we do not control the tokenizer or embedding space of these models, introducing entirely new signpost symbols is not straightforward. Instead, we use natural numbers as signpost-like markers: in our tasks these numbers do not otherwise appear in the input vocabulary, so they provide a simple way to introduce structured intermediate markers while remaining largely disjoint from the task tokens.
We evaluate Boolean Evaluation and S5 Permutation on 3 open-source LLMs spanning different model families and scales: \textsc{LLama-3.3-70B-Instruct-Turbo}, \textsc{Mistral-Small-24B-Instruct-2501}, \textsc{Qwen-2.5-7B-Instruct-Turbo}. We use greedy decoding throughout, and provide $5$ in-context examples in every prompt. Evaluation is done on $100$ samples at each of the lengths $10, 15, 20, 25, 30$. Figure~\ref{fig:prompting-experiments} shows that insights about CoT formatting can help improve performance across the board. On both our tasks, signpost-token prompting generally improves performance on all LLMs, with gains becoming more pronounced at larger lengths. In fact not using signposts sometimes leads to worse performance than random baselines on longer lengths owing to failure to follow instructions. Prompts with signposts do not face such an issue, providing further evidence of the utility of how keeping in mind the learnability of the underlying CoT can result in better CoT formats and in turn better performance, even without any task-specific finetuning.

\textbf{Value change:} To further show the universality of the value change idea, we consider the task of evaluating the value of a variable after various steps of code using LLMs (as depicted in Figure~\ref{fig:barrier-hero-fig}). Given a snippet of code containing many lines, where each line has a variable being assigned to a constant value, we ask what the final value of a subset of these variables is. We ask the LLM to output the final answer without any intermediate reasoning with the goal of showing that our value change strategy can be useful in formatting not just the input itself. We evaluate the same three LLMs by asking for the final answer with simply this code snippet, and with code annotated with comments that mention both the old and the new value of that variable (e.g., \texttt{\# node\_b: 1 -> 0}). As seen in Figure~\ref{fig:prompting-experiments}, having access to such value change has clear benefits for all our models, and each of them achieve much better performance compared to not having any such log.  
Our results suggest that these ideas of signposts and value change can have practical implications for LLMs. We think that practitioners could consider training LLMs with some CoT traces that take these theoretically principled suggestions into account to achieve better performance.

 \section{Discussion}
\label{sec:discussions}

\textbf{Comparison to Prior Constructions for Learnability of CoT.} \label{subsec:prior_turing}
Our work most closely relates to \citet{jiang2026softmax}, who obtained the first length-generalizable Turing completeness construction for Transformers, by extending C-RASP with \emph{custom relative positional encodings}. We complement those results in two ways: (i) we target simulation of TMs and not Counter Machines as done in \citet{jiang2026softmax}. This allows us to make statements more closely linked to most prior constructions, and establish a direct link between TM time complexity and CoT length. In contrast, \citet{jiang2026softmax}'s construction scales with the time complexity of Counter Machines, which may be higher because it relies on unary encoding as an intermediate step; (ii) we use APE and NoPE whereas \citet{jiang2026softmax} uses custom PEs that have not yet been used to train language models.  
\citet{hou2025universal} propose \emph{Turing Programs}, a CoT strategy supporting length generalization. 
Importantly, they disallow $n$-gram repetition in both input and CoT. While this mitigates the difficulty of copying, it also restricts the generality of their strategy, and is incompatible with full Turing completeness.
Finally, \citet{huang2025transformersprovablylearnchainofthought} prove length generalization with gradient descent for CoT on permutation tracking. Crucially, they also require the presence of a growing infinite alphabet with identifiers for individual steps, making their work complementary to ours.

\textbf{Relevance for Practical CoT Reasoning.} We now discuss the practical takeaways of our theoretical results.
The negative result (Theorem~\ref{thm:finite-alphabet-barrier}) states that
only problems in $\TCzero$ can be solved by C-RASP CoT programs. Under the caveats discussed at the end of Section~\ref{subsec:tc0result}, this explains why it is difficult to get transformers to length-generalize over nontrivial CoTs in practice. Our positive result is harder to map to practical situations: Theorem~\ref{thm:positive} assumes an infinite signpost alphabet, and Corollary~\ref{cor:cot-positive} is about a potentially infinite sequence of transformers that are ``trained'' in a very particular way. While the empirical results of Section~\ref{subsec:scratch_experiments} illustrate that length generalization can be boosted with ordinary gradient descent (GD) learning, it remains an essential question for future research how GD can achieve this reliably. Furthermore, real transformers do not have infinite token alphabets. However, every real transformer has a fixed context window and can therefore only ingest instances up to a fixed length $N$. A signpost alphabet of size $N/2$ is sufficient to index these positions, which makes length-generalizing training of real transformers on finite data theoretically feasible. This resonates with earlier empirical findings that larger token vocabularies can improve performance \citep{tao2024scaling,  takase2025large}. Finally, Theorem~\ref{thm:positive} relies on a  specific CoT scheme tracking value changes of tape cells. It is conceivable that earlier empirical work would have found better length generalization with a value-change CoT scheme. However, actually defining value-change CoT for real-world problems, for which it is not obvious what a Turing machine would look like, can be tricky. We still hope that our results will facilitate the design of effective CoT schemes in the future.

\section{Conclusion}
 In this paper, we revisit the role of CoT in Transformer reasoning. If we consider length-generalizable learnability (unlike  expressivity) -- we show that the learnable regime remains bounded by \(\TCzero\), under standard PEs and a finite alphabet. In contrast, with a growing alphabet, Transformers can simulate TMs. Our construction uses signpost tokens to overcome the problem of repeated copying and value-change encodings to avoid non-unique retrievals. These ideas also provide practical principles for designing CoTs, improving performance on known `hard' problems, suggesting that taking  CoT learnability into account during formatting can improve performance at large. 

\section*{Acknowledgments}
We thank Aleksandra Bakalova, Xinting Huang, Entang Wang, Andy Yang for feedback on the paper. Oliver Kraus is funded by the Deutsche Forschungsgemeinschaft (DFG, German Research Foundation) -- GRK 2853/1 “Neuroexplicit Models of Language, Vision,
and Action" - under project number 471607914. Yash Sarrof gratefully acknowledges the stimulating research environment provided by the same. Funded in part by the Deutsche Forschungsgemeinschaft (DFG, German Research Foundation) – Project number  560456343.

\section*{Declaration of LLM Usage} 
We used LLMs to help us polish the writing, find typos and inconsistencies. We also used LLMs to improve the aesthetic of the figures in the paper. Having said that, we take full responsibility for all the content in the paper.

\bibliography{bibliography}
\bibliographystyle{colm2026_conference}

\appendix

\tableofcontents

\section{Appendix}

\subsection{Sample {\CRASP} CoT program for PARITY}
\label{sec:crasp_cot_example}
\par\medskip
\noindent
\begin{tcolorbox}[
    width=\linewidth,
    colback=boxbg,
    colframe=boxframe,
    boxrule=0.5pt,
    arc=2mm,
    left=1.5mm,
    right=1.5mm,
    top=1mm,
    bottom=1mm,
    title={\textbf{A CoT C-RASP program for PARITY = $0^*\{10^*10^*\}^*$}},
    coltitle=black,
    fonttitle=\small\bfseries,    
]
\scriptsize
\renewcommand{\arraystretch}{1.08}
\[
\begin{array}{r@{\quad}l}

\multicolumn{2}{l}{
{\color{commentgray}\text{\scriptsize\sffamily Count of each symbol in the prefix, for }\sigma\in\{1,E,O,\langle SEP\rangle\}}
}
\\[0.2ex]
(1) &
C_\sigma(i) := \cttn{j\le i}{Q_\sigma(j)}
\\[0.6ex]

\multicolumn{2}{l}{
{\color{commentgray}\text{\scriptsize\sffamily Total number of generated parity-state tokens so far}}
}
\\[0.2ex]
(2) &
\mathrm{C}_{\mathrm{par}}(i) := C_E(i)+C_O(i)
\\[0.6ex]

\multicolumn{2}{l}{
{\color{commentgray}\text{\scriptsize\sffamily Is the immediately previous token a given symbol, for }\sigma\in\{E,O,\langle SEP\rangle\}}
}
\\[0.2ex]
(3) &
\mathrm{Prev}_\sigma(i) := \cttn{j\le i,\; j=i-1}{Q_\sigma(j)} \ge 1
\\[0.6ex]

\multicolumn{2}{l}{
{\color{commentgray}\text{\scriptsize\sffamily Initial phase: no parity-state token and no separator has been generated yet}}
}
\\[0.2ex]
(4) &
\mathrm{Init}(i) := (C_{\mathrm{par}}(i)=0)\land(C_{\langle SEP\rangle}(i)=0)
\\[0.6ex]
\multicolumn{2}{l}{
{\color{commentgray}\text{\scriptsize\sffamily Keep alternating the parity state until  }\#1(x)+1\text{\scriptsize\sffamily\ parity tokens have been produced}}
}
\\[0.2ex]
(5) &
\mathrm{Flip}(i) := (1\le C_{\mathrm{par}}(i))\land(C_{\mathrm{par}}(i)\le C_1(i))\land(C_{\langle SEP\rangle}(i)=0)
\\[0.6ex]

\multicolumn{2}{l}{
{\color{commentgray}\text{\scriptsize\sffamily Once the full parity trace has been generated, emit }\langle SEP\rangle\text{\scriptsize\sffamily\ to mark the start of the final answer}}
}
\\[0.2ex]
(6) &
\mathrm{EmitSEP}(i) := (C_{\mathrm{par}}(i)=C_1(i)+1)\land(C_{\langle SEP\rangle}(i)=0)
\\[0.6ex]

\multicolumn{2}{l}{
{\color{commentgray}\text{\scriptsize\sffamily After-separator phase: if the current last token is }\langle SEP\rangle\text{\scriptsize\sffamily , the next token should be the final parity answer}}
}
\\[0.2ex]
(7) &
\mathrm{AfterSEP}(i) := Q_{\langle SEP\rangle}(i)
\\[0.6ex]

\multicolumn{2}{l}{
{\color{commentgray}\text{\scriptsize\sffamily End phase: after outputting the final answer token immediately following }\langle SEP\rangle\text{\scriptsize\sffamily , emit }\langle EOS\rangle}
}
\\[0.2ex]
(8) &
\mathrm{EmitEOS}(i) := (Q_E(i)\lor Q_O(i))\land \mathrm{Prev}_{\langle SEP\rangle}(i)
\\[0.6ex]

\multicolumn{2}{l}{
{\color{commentgray}\text{\scriptsize\sffamily Output }E\text{\scriptsize\sffamily : either start with }E\text{\scriptsize\sffamily , flip from }O\text{\scriptsize\sffamily\ to }E\text{\scriptsize\sffamily\ while building the trace, or copy the final even answer after }\langle SEP\rangle}
}
\\[0.2ex]
(9) &
\operatorname{OUTPUT}(E) := \mathrm{Init}(i)\lor(\mathrm{Flip}(i)\land \mathrm{Prev}_O(i))\lor(\mathrm{AfterSEP}(i)\land \mathrm{Prev}_E(i))
\\[0.6ex]

\multicolumn{2}{l}{
{\color{commentgray}\text{\scriptsize\sffamily Output }O\text{\scriptsize\sffamily : either flip from }E\text{\scriptsize\sffamily\ to }O\text{\scriptsize\sffamily\ while building the trace, or copy the final odd answer after }\langle SEP\rangle}
}
\\[0.2ex]
(10) &
\operatorname{OUTPUT}(O) := (\mathrm{Flip}(i)\land \mathrm{Prev}_E(i))\lor(\mathrm{AfterSEP}(i)\land \mathrm{Prev}_O(i))
\\[0.6ex]

\multicolumn{2}{l}{
{\color{commentgray}\text{\scriptsize\sffamily Output the separator exactly when the chain-of-thought trace is complete}}
}
\\[0.2ex]
(11) &
\operatorname{OUTPUT}(\langle SEP\rangle) := \mathrm{EmitSEP}(i)
\\[0.6ex]

\multicolumn{2}{l}{
{\color{commentgray}\text{\scriptsize\sffamily Output }\langle EOS\rangle\text{\scriptsize\sffamily\ immediately after the final parity answer}}
}
\\[0.2ex]
(12) &
\operatorname{OUTPUT}(\langle EOS\rangle) := \mathrm{EmitEOS}(i)

\end{array}
\]
\end{tcolorbox}
\par\medskip

\subsection{Limitations of C-RASP CoTs}\label{app:theorem-1}

It was shown in \citet{jiang2026softmax} (their Proposition 4.1) that any function with super-polynomial NFA complexity has no C-RASP CoT; the same applies to $\CRASPplshort$.
This implies that certain tasks, such as copying a long string with repetitions, are out of reach for C-RASP even when CoT is allowed.
However, the result doesn't rule out C-RASP CoTs for regular languages.
Here, we show a result that even has implications for regular languages:

\begin{theorem}[Restated from Theorem~\ref{thm:finite-alphabet-barrier}]
Let $\mathcal{L} \subset \Sigma^*$, and let $F_\mathcal{L} : \Sigma^* \rightarrow \{1, 0\}$ be its indicator function (i.e., $F_\mathcal{L}(w) = 1 \Leftrightarrow w \in \mathcal{L}$). If $F_\mathcal{L}$ has a CoT in $\CRASPplshort$, then $\mathcal{L}$ is definable in $\TCzero$.
\end{theorem}
Remarkably, this holds irrespective of the length of the CoT.

\begin{proof}
The proof idea is an adaptation of Theorem 12 in \citet{huangFormalFrameworkUnderstanding2025}. Previously, \citet{jiang2026softmax} had used that theorem to show that C-RASP CoTs are not Turing-complete. Here, we strengthen this to showing membership in $\TCzero$.

    We give the input to Alice, and have Bob compute the CoT.   
Let $\mathcal{C}$ be the set of all count terms in the C-RASP program; each of the form $C(i) := \# [j \leq i, \phi(j,i)] \psi(j) \wedge P(j)$ where $\phi$ is local or $\top$, and $\psi$ is periodic or $\top$.
Let $L$ be the length of the input.
In order to allow Bob to compute the CoT, it is sufficient for Alice to communicate $C(L-l), \dots, C(L)$ for each $C \in \mathcal{C}$, where $l$ is a bound determined by the set of local functions included.

We now build a $\TCzero$ circuit deciding $\mathcal{L}$.
We compute the binary representation of each $C(L-l), \dots, C(L)$ for each $C \in \mathcal{C}$.
Each of these has $O(\log N)$ bits, when $N$ is the input length.
We thus obtain a string of length $O(l \cdot |\mathcal{C}| \cdot \log N)$.
By assumption, we can determine Bob's output (and thus membership in $\mathcal{L}$) from this string.
As the set of such strings has cardinality polynomial in $N$, we can hard-code a lookup table from this string to Bob's output while staying in $\TCzero$.

Alice can communicate $O(\log N)$ bits that are sufficient for Bob to compute the CoT.
As these are counts definable in C-RASP, these bits are computable in $\TCzero$.
Now, a $\TCzero$ circuit can hard-code the mapping from these $O(\log N)$-bit-length messages to the membership labels that result at the end of the CoTs.
Hence, these labels are computable by $\TCzero$ circuits.
\end{proof}

\begin{lemma}[Failure of Length Generalization outside of $\TCzero$]
    When $\mathcal{L}$ is not in $\TCzero$, then, for some $f_{TestLength} : \mathbb{N} \rightarrow \mathbb{N}$, under CoT learning notion of Definition~\ref{def:cot-learnable-finite} based on the learning model of \citet{huangFormalFrameworkUnderstanding2025}, length generalization from training length $n$ to test length $f_{TestLength}(n)$ will fail for infinitely many $n$.
\end{lemma}

\begin{proof}
\cite{huangFormalFrameworkUnderstanding2025} show that, in the idealized learning model, length generalization for arbitrary test lengths holds if and only if the function is expressible by \emph{Limit Transformers}, an idealized computational model used for modeling the learnability of standard transformers.
While C-RASP[Pos] can be translated to Limit Transformers, a converse is not known; nonetheless we now show that a result analogous to Theorem~\ref{thm:finite-alphabet-barrier} holds if we start not from C-RASP, but from Limit Transformers: for these, as shown in the proof of Theorem 12 in \cite{huangFormalFrameworkUnderstanding2025}, Alice can likewise summarize all information relevant to Bob into $O(\log N)$ bits. We now argue that Limit Transformers can be simulated by log-precision transformers.
First, this is clear for the feedforward networks, which only use ReLU and Heaviside activations. It is also clear for the $\phi(\cdot,\cdot)$ functions used in the attention logits of limit transformers.
Furthermore, due to the rounding used in the attention logits of the definition of limit transformers, attention scores are also computable in log-precision.
Thus, overall, Limit Transformers can be simulated by log-precision transformers, and thus are in $\TCzero$ \citep{merrill2023parallelism}.
\end{proof}

\subsection{Operations in {\CRASP} and {\infrasp}}

\emph{$\infrasp$} has an extra predicate over $\CRASPplshort$ \citep{huangFormalFrameworkUnderstanding2025} which acts over the infinite alphabet $\mathcal{C}$ which in fact can be assumed without loss of generality to be congruent to Natural numbers (and thus have a natural ordering) $\mathcal{C} \cong  \mathbb{N}$. This predicate, termed  a \textbf{Match Predicate} $\chi(i,j)$   checks for equality between tokens in the neighborhoods of $i$ (the current position) and (some past position) $j$. Formally,  $\chi(i,j) := \bigwedge_{k=1}^K (\vc_{j-\delta_k} = \vc_{i-\gamma_k} + \tau_k)$.  Here, $\delta_k, \gamma_k \in \mathbb{N}$ and $\tau_k \in \mathbb{Z}$ are small constants, $c_i$ is the token at position $i$. $\chi(i,j)$ is True if position $j-\delta_k$ and $i-\gamma_k$ represent the same number (where position $i$ could be shifted by a small number $\tau_k$)) for all $k$.

This $\chi(i, j)$ matches tokens based on local constants $\delta, \gamma, \tau$, without  memorizing identities of any token from $\mathcal{C}$. It performs simple pattern matches in local neighborhoods throughout the context window. 

On top of adding this predicate to $\CRASPplshort$, $\infrasp$ also removes the periodic operation $\phi(i)$ which could be set by the absolute positional embeddings in Transformers and allowed a setting of a periodic value based on the current position of the token, instead of the value at the token (\cite{sarrof2026planning} provides a discussion behind these changes). A $\infrasp$ program $P$ formally consists of a sequence $P_1,\ldots, P_k$ of operations, where each operation is either Boolean-valued or count-valued. (Full list of rules in Table~\ref{tab:infrasp-ops}). It should be noted that the alphabet of this program is $\Sigma \cup \mathcal{C}$, where $\Sigma$ is finite, and $\mathcal{C} \cong \mathbb{N}$.  Crucially, only $\Sigma$ can be found with unary queries $Q_\sigma$, while symbols from $\mathcal{C}$ can be matched only through $\chi(i, j)$. 
$\Psi$ is a set of \emph{binary relations} $\psi : \mathbb{N} \times \mathbb{N} \rightarrow \{0,1\}$. Counting operations return the number of positions $j\leq i$ where $P(j)$ and 
$\psi(i,j)$ which represents functions such as $i = j + \delta$, for a fixed $\delta$, thus acting as a local relation ($\delta$ cannot be arbitrarily large) for checking things in local neighborhoods. 
If a task has a $\infrasp$ program, then it is guaranteed to length generalize on it (under the learning model of \citet{sarrof2026planning}) for APE transformers. If no $\psi \in \Psi$ is used and all the constants in the match predicate are $0$, then the guarantee holds for NoPE transformers as well. 

\begin{table}[t]
\centering
\small
\begin{tabular}[t]{@{}l l@{}}
    \toprule
    \multicolumn{2}{c}{\textbf{Boolean-Valued Operations}} \\
    \midrule
    \textbf{Initial} & $P(i):=Q_\sigma(i)$ \quad for $\sigma\in\Sigma$\\
    \addlinespace
    \textbf{Boolean} & $P(i):=\lnot P_1(i)$ \\
                     & $P(i):=P_1(i)\land  P_2(i)$\\
    \addlinespace
    \textcolor{red}{\textbf{Positional}} & $P(i):= \phi(i)$ for $\phi \in \Phi$\\
    \addlinespace
    \textbf{Constant}& $P(i):=\top$\\
    \addlinespace
    \textbf{Comparison} & $P(i):=C_1(i)\leq C_2(i)$\\
    \bottomrule
\end{tabular}%
\hspace{2em}
\begin{tabular}[t]{@{}l l@{}}
    \toprule
    \multicolumn{2}{c}{\textbf{Count-Valued Operations}} \\
    \midrule
    \textbf{Counting} & $C(i):=\cttn{j\leq i, \psi(i,j)}{P(j)}$ \\
                      & \quad for $\psi\in\Psi\cup\{\top\}$\\
    \addlinespace
    \textcolor{green}{\textbf{Match}} & $C(i) := \cttn{j \le i, \chi(i,j)}$ \\
    \addlinespace
    \textbf{Conditional} & $C(i):=\cif{P(i)}{C_1(i)}{C_2(i)}$\\
    \addlinespace
    \textbf{Add/Subtract}    & $C(i):=C_1(i) \pm C_2(i)$\\
    \addlinespace
    \textbf{Constant}    & $C(i):=1$\\
    \bottomrule
\end{tabular}
\caption{Operations in $\CRASPplshort$  and $\infrasp$ programs. The operations in \textcolor{red}{red} and \textcolor{green}{green} denote the deletion and addition of operations from $\CRASPplshort$ to $\infrasp$.}
\label{tab:infrasp-ops}
\end{table}

\subsection{Turing Completeness of {\infrasp}}

\begin{theorem} \label{thm:positive_restate}
\textit{($C^*$-RASP with CoT is Turing-complete)}
Let $F$ be a semi-computable function computed by a Turing machine $\mathcal{T}$.
Then there is a $C^*$-RASP program $S$ computing it.

\textit{($C^*$-RASP CoTs have length bounded by Turing machine time complexity)}
Furthermore, there is a universal constant $C$ such that, when $\mathcal{T}$ terminates on $w$ in $N$ steps, then $S$ generates a string of length $C \cdot N_{tapes} \cdot (N + |w|)$ ending in $\langle EOS \rangle$ on $Annotate(w)$, where $N_{tapes}$ is the number of tapes of $\mathcal{T}$.
\end{theorem}

\begin{proof}
We explain the proof in the special case of a single-tape Turing machine to avoid notational clutter.
For the more general case of a multi-tape Turing machine, it suffices to add identifiers for the (finitely many) tapes to $\Gamma$ and pair each value change log with a tape identifier indicating which tape the change applies to, and always keep track of the head position for each tape. The last part leads to an increase of the length of each block by a factor of $N_{tapes}$.

We assume the convention that the output of the TM is $\bot$ when it does not terminate, and the final nonblank prefix of a designated tape (or the single tape, if there is just one) at termination.

Let $F:\Sigma^*\to\Sigma^*\cup\{\bot\}$ be a semi-computable function. We define our  deterministic Turing machine as follows
\[
\mathcal T=(Q,\Gamma,\sqcup ,\delta,q_0,H)
\]
computing $F$, where:
\begin{enumerate}
    \item $Q$ is a finite set of states;
    \item $\Gamma$ is a finite tape alphabet with $\Sigma\subseteq \Gamma$;
    \item $\sqcup \in\Gamma\setminus\Sigma$ is the blank symbol;
    \item $q_0\in Q$ is the initial state;
    \item $H\subseteq Q$ is the set of halting states;
    \item $\delta:(Q\setminus H)\times \Gamma \to Q\times \Gamma\times\{L,R,S\}$
    is the transition function ($L, R, S$ correspond to moving left, right or staying).
\end{enumerate}

We use a one-sided tape indexed by \(\mathbb N_{>0}=\{1,2,3,\dots\}\), and every head initially points to position \(1\). On input $w=w_1\cdots w_n\in\Sigma^*$
the tape initially contains \(w_1,\dots,w_n\) in cells \(1,\dots,n\), and every other cell initially contains \(\sqcup \).
\footnote{We fix the one-sided-tape convention that if the head is at cell \(1\) and the transition prescribes a left move, the head remains at cell \(1\).}

We construct a CoT \(C^*\)-RASP program \(P\) which, on input \(\mathsf{Annotate}(w)\), simulates the run of \(\mathcal T\) on \(w\). If \(\mathcal T\) halts, then \(P\) eventually generates $\langle SEP\rangle\,\mathsf{Annotate}(F(w))\,\langle EOS\rangle$.
If \(\mathcal T\) does not halt, then \(P\) never generates \(\langle EOS\rangle\).

\paragraph{Finite and infinite parts of the CoT alphabet.}
Recall that the CoT in the infinite-alphabet setting generates tokens from a fixed finite alphabet together with the increasing alphabet \(\mathcal C\) of signpost tokens. In this construction we use $\mathcal C=\{c_1,c_2,c_3,\dots\}$ The finite control alphabet \(\Xi\) is chosen to contain:
\begin{enumerate}
    \item all state symbols \(q\in Q\);
    \item the special symbols \(\langle SEP\rangle\), \(\langle EOS\rangle\), 
    and a block separator \(\$\);
    \item the symbol $\mathrm{KEEP}$ to show that no operation is taking place on the cell. 
    \item for each pair of distinct symbols \(\sigma,\tau\in\Gamma\), a symbol
    $\mathrm{WRITE}(\sigma\to\tau)$. 
    
\end{enumerate}
Thus the generated CoT is a string over $\Xi\cup\mathcal C$.

\paragraph{Annotated input.}
For $w=w_1\cdots w_n\in\Sigma^*$,
we use the annotated encoding
$\mathsf{Annotate}(w):=w_1c_1\,w_2c_2\cdots w_nc_n$.
Hence each input symbol is followed immediately by its signpost token.

\paragraph{Notation.} For each simulated time step \(m\ge 0\) 

\begin{itemize}

\item Let $\eta_{m}\in\mathbb N_{>0}$ be the position of the head after \(m\) steps. 

\item Let $q_m\in Q$ be the machine state after \(m\) steps

\item Let $\xi_{m}(r)\in\Gamma$
be the symbol stored in cell \(r\) after \(m\) steps.
\end{itemize}

\paragraph{Block format.}
For each simulated step \(m\ge 0\), the CoT contains one block
\[
B_m :=
q_m\;
c_{\eta_{m}}\,\$ \, e_{m}\;
\]
where
\[
e_{m}=
\begin{cases}
\mathrm{KEEP}, & \text{transition at step }m\text{ leaves the symbol on the cell }  c_{\eta_{m}} \text{ unchanged},\\[1mm]
\mathrm{WRITE}(\sigma\to\tau), & \text{transition at step }m\text{ changes the symbol on the cell } c_{\eta_{m}}\text{ from }\sigma\text{ to }\tau,
\end{cases}
\]
with \(\sigma\neq\tau\) in the second case.

The tape contents are therefore not stored explicitly. Instead, the current content of a tape cell is reconstructed from its initial value together with all logged value-change events associated with the same index symbol.

\paragraph{Current state.}
For each \(q\in Q\), define
\[
\mathrm{CURR\mbox{-}STATE}_q(i)
:=Q_{\$}(i)\wedge Q_q(i-2).
\]
Exactly one of these predicates holds at the each $\$$ in a block.

\paragraph{Matching the current head cell.}
Define the match predicate
\[
\chi(i,j):=\bigl(c_{j-2}=c_{i-1}\bigr).
\]
Intuitively, \(\chi(i,j)\) says that the event encoded at position \(j\) happened to the same tape cell as the current cell in the current block whose $\$$ lies at position $i$. The offsets here are $2$ and $1$ because the cell identity is $2$ positions away from the write events and cell identity is $1$ position away from the $\$$ symbol.

\paragraph{Counting value-change events.}
For each each pair of distinct symbols \(\sigma,\tau\in\Gamma\), define
\[
N_{\sigma\to\tau}(i)
:=
\# [j\le i,\chi(i,j) \land  Q_{\mathrm{WRITE}(\sigma\to\tau)}(j)].
\]
Thus \(N_{\sigma\to\tau}(i)\) counts the number of logged write-events of type \(\sigma\to\tau\) at the cell currently under our head.

\paragraph{Initial contents.}
We next define, for each \(\tau\in\Gamma\), a Boolean predicate \(\mathrm{INIT}_{\tau}(i)\) which evaluates to true if the tape cell currently under the head initially contained \(\tau\).

The initial content is determined by the annotated input. Thus, if we are evaluating at the $\$$ token, we want to check if the cell token at position $i-1$ when it appeared in the input (say $j$), which token it had, i.e. if it matched some input character at position $j-1$. First we need to match the cell token itself. Thus, 
\[
\chi_{\mathrm{inp}}(i,j):=(c_{j}=c_{i-1}).
\]

Thus, this says that the index token at input position \(j\) is the same as the current head cell. Now, for each \(\tau\in\Sigma\), we check if the symbol right before the current position is a specific character.   

\[ \text{PrevPosIsSymbol}_\tau(i) :=\cttn{i=j+1}Q_\tau(j)\]

Thus, our overall check is simply the position where both of these things are true.

\[
\mathrm{INIT}_{\tau}(i)
:=
Q_{\$}(i)\wedge
\left(
\# [j\le i,\chi_{\mathrm{inp}}(i,j) \land  \text{PrevPosIsSymbol}_\tau(j)]
\right)\ge 1.
\]
Since each input index token occurs at most once in \(\mathsf{Annotate}(w)\), this holds exactly when the current head position on our tape corresponds to an input cell whose symbol is \(\tau\).
\footnote{For the blank symbol, define \(\mathrm{INIT}_{\sqcup}(i)\) to hold exactly when the current head marker \(c_{i-1}\) does not occur in the annotated input, i.e.\ when no input signpost matches the current cell. This ensures that cells outside the original input are initialized to \(\sqcup\).}

\paragraph{Recovering the current symbol under the head.}
For each \(\tau\in\Gamma\), define
\[
I_{\tau}(i):=\cif{\mathrm{INIT}_{\tau}(i)}{1}{0}
\]
and
\[
B_{\tau}(i)
:=
I_{\tau}(i)
+
\sum_{\sigma\in\Gamma\setminus\{\tau\}} N_{\sigma\to\tau}(i)
-
\sum_{\sigma\in\Gamma\setminus\{\tau\}} N_{\tau\to\sigma}(i).
\]
Intuitively, \(B_{\tau}(i)\) is the net balance of symbol \(\tau\) at the current head cell: it starts with the initial symbol, gains one for every change into \(\tau\), and loses one for every change out of \(\tau\).

Now define
\[
\mathrm{CURR\mbox{-}SYMB}_{\tau}(i)
:=
Q_{\$}(i)\wedge (B_{\tau}(i)=1).
\]

The key invariant is that, exactly one \(\tau\in\Gamma\) satisfies \(\mathrm{CURR\mbox{-}SYMB}_{\tau}(i)\), namely the actual current symbol stored at the head position.
\paragraph{Predicates for the next transition.}
Since \(\delta\) is finite, all information about the next transition can be obtained by an explicit finite case distinction over the current state and the current symbol under the TM head. To enhance readability, we will slightly abuse the notation, and assume that each of the $\bigvee$ symbols in the next few predicates all have the same suffix, $q\in Q\setminus H,\ \tau \in\Gamma ; \quad  
\delta(q,\tau)=(q',\tau',d)$ where $d \in \{\text{L,R,S}\}$.

For each \(q'\in Q\), define
\[
\mathrm{NEXT\mbox{-}Q}_{q'}(i)
:=
Q_{\$}(i)\wedge
\bigvee
\left(
\mathrm{CURR\mbox{-}STATE}_q(i)\wedge
\mathrm{CURR\mbox{-}SYMB}_{\tau}(i)
\right)
\]
Since our transition function is deterministic, \(\mathrm{NEXT}_{q'}(i)\) holds exactly when the next state of the Turing machine \(q'\) is computed from the current state $q$ and value $\tau$ under the TM head, and all the other conditions within the $\bigvee$ will evaluate to false. 

Similarly, we define our writes
\[
\mathrm{NEXT}_{\mathrm{KEEP}}(i)
:=
Q_{\$}(i)\wedge
\bigvee_{\tau'=\tau}
\left(
\mathrm{CURR\mbox{-}STATE}_q(i)\wedge
\mathrm{CURR\mbox{-}SYMB}_{\tau}(i)
\right),
\]
and, for each pair of distinct symbols \(\tau, \sigma \in\Gamma\),
\[
\mathrm{NEXT}_{\tau\to\sigma}(i)
:=
Q_{\$}(i)\wedge
\bigvee_{\tau'=\sigma}
\left(
\mathrm{CURR\mbox{-}STATE}_q(i)\wedge
 \mathrm{CURR\mbox{-}SYMB}_{\tau}(i)
\right).
\]
Thus \(\mathrm{NEXT}_{\mathrm{KEEP}}(i)\) says that the transition leaves the current symbol unchanged, whereas \(\mathrm{NEXT}_{\tau\to\sigma}(i)\) says that the cell changes its current symbol from \(\tau\) to \(\sigma\).

Likewise, for each direction \(D\in\{L,R,S\}\), define
\[
\mathrm{MOVE}_{D}(i)
:=
Q_{\$}(i)\wedge
\bigvee_{d=D}
\left(
\mathrm{CURR\mbox{-}STATE}_q(i)\wedge
\mathrm{CURR\mbox{-}SYMB}_{\tau}(i)
\right).
\]

Hence, at the end of a completed block, the predicates \(\mathrm{NEXT}_{q'}\), \(\mathrm{NEXT}_{\mathrm{KEEP}}\), \(\mathrm{NEXT}_{\tau\to\sigma}\), and \(\mathrm{MOVE}_{D}\) are all definable by finite disjunctions over the transition table of \(\mathcal T\).

\paragraph{Generating the next block token by token.}
We now explain how these predicates are used to generate the next block in the CoT.

A completed block has the form $q_m c_{\eta_{m}} \$ e_{m}\;$.

Since, we can only the $e_{m}$ the event that should happen at the current cell after the transition has happened. Thus at the $\$$ token, we will autoregressively first generate the new write action, and then the new state and the new cell position token, ending with the generation of the $\$$ token itself. These steps loop until we get to a halting state. 

$q_m c_{\eta_{m}} \$ \to q_m c_{\eta_{m}} \$e_m \to q_m c_{\eta_{m}} \$e_mq_{m+1} \to q_m c_{\eta_{m}} \$e_mq_{m+1}c_{\eta_{m+1}} \to q_m c_{\eta_{m}} \$e_mq_{m+1}c_{\eta_{m+1}}\$$ 

We will now use our output clauses to generate these tokens.

\begin{itemize}
    \item \emph{New Write Action:}
We generate the new write action at our current head first i.e. the token \(e_{m}\). We evaluate at our $\$$ token, and hence we can evaluate our relevant predicates \(\mathrm{NEXT}_{\mathrm{KEEP}}\) and \(\mathrm{NEXT}_{\tau\to\sigma}\) correctly there and output the next token correctly.

Thus, we include the output clauses
\[
\operatorname{OUTPUT}(\mathrm{KEEP})\defeq 
(Q_\$(i) \land  \mathrm{NEXT}_{\mathrm{KEEP}}(i))
\]
and
\[
\operatorname{OUTPUT}(\mathrm{WRITE}_{\tau\to\sigma})
\defeq
(Q_\$(i) \land  \mathrm{NEXT}_{\tau\to\sigma}(i)),
\]

Again, because \(\delta\) is deterministic, exactly one such write clause can ever be active.

\item \emph{State slot:}
Next, we want to output the new state, after having computed it already at our last $\$$ token, we are 2 positions away from where we want to output the new state, and thus we use our output clauses at offset -1 to correctly output everything in the right position. Thus, for each \(q'\in Q\), we have the output clause
\[
\operatorname{OUTPUT}(q') \defeq \cttn{j \le i ; i = j + 1} (Q_{\$}(j)\wedge \mathrm{NEXT}_{q'}(j)).
\]
Since exactly one next state is prescribed by \(\delta\), exactly one such clause is active. 

\item \emph{Signpost slot:} Similar to the others, 

The next head marker is determined by the move predicates:
\begin{align*}
\operatorname{OUTPUT\_SIGNPOST}_{(2, -1)} &\defeq \cttn{j \le i ; i = j + 2}(Q_\$(j) \wedge \mathrm{MOVE}_{L}(j))\\
\operatorname{OUTPUT\_SIGNPOST}_{(2, 0)}  &\defeq \cttn{j \le i ; i = j + 2}(Q_\$(j) \wedge \mathrm{MOVE}_{S}(j))\\
\operatorname{OUTPUT\_SIGNPOST}_{(2, 1)} &\defeq \cttn{j \le i ; i = j + 2}(Q_\$(j) \wedge \mathrm{MOVE}_{R}(j))
\end{align*}
Here, the move direction dictates whether we pick the same signpost token located at an offset of distance 2 from our current position (Stay) or increment it (right) or decrement it (left). However, relative to the last position, the position at an offset of 2 is fixed, and the move direction whether left, right or staying would also be fixed depending on the transition functions output, we will be able to uniquely output the required next signpost token.  

\item \emph{Separator slot:}
After all tape slots have been generated, the next token is \(\$\). We always generate this every 4 tokens. 
\[
\operatorname{OUTPUT}(\$)\defeq \cttn{j \le i ; i = j + 3}Q_\$(j),
\]

\end{itemize}

Our first block can be generated directly from the input. The initial state is \(q_0\), and the head starts at position \(1\), so the first cell head is always \(c_1\). The predicates \(\mathrm{INIT}_{\tau}\) recover the correct initial content from \(\mathsf{Annotate}(w)\). Since no write-events have yet been logged, all counts \(N_{\sigma\to\tau}\) are zero. Hence the value change balance \(B_{\tau}\) recovers the correct initial tape contents, and our invariant holds.

Assume that we have been able to generate  the state and the current TM head cell signpost for the block \(B_m\) and have generated the separator token.  
Through our value change balance construction, we can recover exactly the true current symbol \(\xi_{m}(\eta_{m})\) under our head. 

If the transition at step \(m\) does not write to cell \(r\), then no new write-event for that cell is logged, so all balances remain unchanged. If the transition changes the symbol at cell \(r\) from \(\sigma\) to \(\tau\), then exactly one new event \(\mathrm{WRITE}(\sigma\to\tau)\) is logged for that cell. This decreases \(B_{\sigma}\) by \(1\), increases \(B_{\tau}\) by \(1\), and leaves all other balances unchanged. Hence the balance construction continues to recover the correct current tape symbol. Therefore the predicates \(\mathrm{NEXT}\), and \(\mathrm{MOVE}\) pick out exactly the unique transitions prescribed by \(\delta\), and we one after the other generate the correct event slot, the correct state slot and also the correct head slot.

Thus, the event slots output exactly the correct symbols \(\mathrm{KEEP}\) or \(\mathrm{WRITE}(\sigma\to\tau)\), thereafter the next state slot correctly outputs the correct new state \(q_{m+1}\). Finally we also output the correct new head positions \(c_{\eta_{m+1}}\), according to whether the transition leaves the current symbol unchanged or changes it from \(\sigma\) to \(\tau\).

\paragraph{Generating the output}
Once a halting state \(q_m\in H\) is reached, the CoT switches from the simulation phase to an output phase: it emits \(\langle SEP\rangle\), then scans the output tape from cell \(1\) to the first blank cell, emitting the annotated string \(\mathsf{Annotate}(F(w))\), and finally emits \(\langle EOS\rangle\) as soon as a blank cell is reached. Throughout, it again uses the value-change computation to read out the content of each cell.
If the machine never halts, this phase is never entered, and no \(\langle EOS\rangle\) is generated.

\paragraph{Number of emitted tokens}
For simulating each step of the TM transition, our construction outputs $4$ CoT tokens. hence if the Turing machine terminates in $N$ steps, then our {\infrasp} program generates $4N$ tokens in the simulation phase.
In the output phase, it generates at most another $2 + 2(N+|w|)$ tokens making up \(\langle SEP\rangle\mathsf{Annotate}(F(w))\langle EOS\rangle\); note here that the length of the nonblank prefix of the tape at the end cannot exceed $N+|w|$.
In the case of a multi-tape Turing machine, the length of each block increases by a factor of $N_{Tapes}$.

\end{proof}

\section{Additional Experimental Details}

\subsection{Task Descriptions}
\label{appendix:task-descriptions}

This section describes the evaluated algorithmic tasks on a high level. For each task we give a high-level example. For the respective trained CoT constructions see section \ref{appendix:cot-training-formats}.

\textbf{Parity}
In the Parity task, the model receives a string consisting of 0s and 1s and must evaluate whether the total number of 1s in the sequence is even or odd. After processing the full sequence, the model outputs either "E" or "O", representing an even or odd final result, respectively. For both the training and evaluation datasets, the individual bits within the sequences are randomly sampled.

\begin{promptbox}{Task Example: Parity}
\textbf{Problem:} Evaluate whether the number of 1s in a sequence of binary digits is even or odd.

\vspace{2mm}
\hrule
\vspace{2mm}

\textbf{Input:} \\
01101100

\vspace{2mm}
\hrule
\vspace{2mm}

\textbf{Expected Output:} \\
Even
\end{promptbox}

\textbf{Boolean Evaluation}
In the Boolean Evaluation task, the model is tasked with evaluating the truth value of a boolean formula. The input consists of a nested sequence of logical operations (i.e., AND, OR, NOT) applied to the primitives "True" and "False". The model must parse the hierarchical structure of the expression and compute the intermediate truth values to arrive at the correct final state. After evaluating the full formula, the model outputs either "True" or "False". For the training and evaluation sets, the Boolean formulas and their corresponding primitives are generated randomly to ensure diverse logical depths and sequence lengths.
\begin{promptbox}{Task Example: Boolean Evaluation}
\textbf{Problem:} Evaluate the final truth value of a nested logical expression.

\vspace{2mm}
\hrule
\vspace{2mm}

\textbf{Input Formula (AST Size = 4):} \\
( NOT true ) AND false

\textbf{Conceptual Resolution:}
\begin{enumerate}
    \item Evaluate inner clause: \texttt{NOT true} $\rightarrow$ \texttt{False}
    \item Evaluate root operator: \texttt{False AND false} $\rightarrow$ \texttt{False}
\end{enumerate}

\vspace{2mm}
\hrule
\vspace{2mm}

\textbf{Expected Output:} \\
False
\end{promptbox}
\textbf{S5 Permutation}
In this task, the model receives a natural language description of an initial state, which assigns five distinct objects to five variables (e.g., A = Apple, B = Banana). This initialization is followed by a sequence of state-changing operations consisting of pairwise "swap" instructions that exchange the contents of two variables. The model must continuously track the shifting positions of all five objects throughout the sequence and output their final arrangement. Task length is defined as the number of pairwise swap operations. Both the initial object assignments and the sequence of swaps are randomly sampled for both the training and out-of-domain evaluation sets.

\begin{promptbox}{Task Example: S5 Permutation}
\textbf{Problem:} Track the locations of 5 distinct items across a series of swap operations.

\vspace{2mm}
\hrule
\vspace{2mm}

\textbf{1. Initialization:} \\
A = Apple, B = Banana, C = Cat, D = Dog, E = Hat

\textbf{2. Operations (Sequence Length = 2):} \\
Step 1: Swap A and E \\
Step 2: Swap B and C

\vspace{2mm}
\hrule
\vspace{2mm}

\textbf{Expected Output (Final State):} \\
A = Hat, B = Cat, C = Banana, D = Dog, E = Apple
\end{promptbox}

\textbf{Binary Permutation}
\label{appendix:binary_permutation_description}
This is the same task as S5 Permutation, but we restrict the object pool to only two objects (Cat, Dog) that are assigned to the five variables. We ensure to generate no problem instance where all variables are initialized to the same object, so the task does not become trivial.

\begin{promptbox}{Task Example: Binary Permutation}
\textbf{Problem:} Track the locations of 2 distinct items across a series of swap operations.

\vspace{2mm}
\hrule
\vspace{2mm}

\textbf{1. Initialization:} \\
A = Cat, B = Dog, C = Cat, D = Dog, E = Dog

\textbf{2. Operations (Sequence Length = 2):} \\
Step 1: Swap A and E \\
Step 2: Swap B and C

\vspace{2mm}
\hrule
\vspace{2mm}

\textbf{Expected Output (Final State):} \\

A = Dog, B = Cat, C = Dog, D = Dog, E = Cat
\end{promptbox}

\subsection{CoT Training Formats \& Further Training Details}\label{appendix:cot-training-formats}
\subsubsection{Parity}
\textbf{Data Generation:} For both training and evaluation, binary sequences of lengths bounded by the respective experimental setups are randomly sampled.

\textbf{CoT Formats:} We compare a standard dense trace against our proposed value-change trace. The \textbf{Naive CoT} outputs the running parity state ('E' or 'O') after processing every single bit in the input sequence. In contrast, the Value-Change CoT starts with an initial 'E' token and only outputs a new state when the parity actually flips (i.e., upon encountering a '1'), skipping '0's entirely. This format trains the model to only log explicit state transitions.

\begin{promptbox}{CoT Training Formats: Parity}

\textbf{Prompt Format:} \\
\texttt{0 0 1 0 1 0 0 <|trace|>}

\vspace{2mm}
\hrule
\vspace{2mm}

\textbf{Naive CoT Trace:} \\
\texttt{E E O O E E E answer E<|endoftext|>}

\vspace{2mm}
\hrule
\vspace{2mm}

\textbf{Value-Change CoT Trace:} \\
\texttt{E O E answer E<|endoftext|>}
\end{promptbox}

\subsubsection{Boolean Evaluation}

\textbf{Data Generation:} For both training and evaluation, we generate random Boolean formulas represented as Abstract Syntax Trees (ASTs) containing the literals \texttt{true} and \texttt{false}, and the operators \texttt{NOT}, \texttt{AND}, and \texttt{OR}. Sequence length is defined by the total number of nodes in the AST. To construct the dataset, we first sample a target sequence length uniformly at random from the length bounds specified by the given experiment. To guarantee that our dataset is structurally unbiased, we uniformly sample from the space of all valid Boolean formulas of that exact size. We achieve this by precomputing the exact combinatorial counts of valid formulas for each size under our grammar, and then recursively sampling operations weighted by these exact counts. This ensures a uniform distribution over all possible AST structures, avoiding the depth and branching biases inherent to naive top-down generation.

\textbf{CoT Formats:} The Naive CoT requires the model to output the truth value ('T' or 'F') of each subformula in standard post-order without any explicit structural pointers. The signpost CoT annotates the input formula with unique signpost tokens for every operation and literal. Each literal is preceded by an index hint.
In addition, an operator is preceded by square brackets which point to the child operand it depends on 
(e.g., \texttt{[<0>, <1>] <2> AND} $\rightarrow$ evaluate the truth values associated with signpost tokens 0 and 1 with the operator AND). The formula can thus be solved by following the signpost tokens in their natural order.
In the trace, the model must prefix every evaluated truth value with its corresponding index, providing explicit retrieval targets for deeply nested clauses. The signpost tokens (numbers in angle brackets) are implemented as atomic tokens in the tokenizer.

\vspace{2mm}
\begin{promptbox}{CoT Training Formats: Boolean Evaluation}
\textbf{Naive CoT Prompt:} \\
\texttt{( false OR NOT ( false AND true ) ) <|trace|>}

\vspace{2mm}
\hrule
\vspace{2mm}

\textbf{Naive CoT Trace:} \\
\texttt{F F T F T T answer: T<|endoftext|>}

\vspace{2mm}
\hrule
\vspace{2mm}

\textbf{Signpost CoT Prompt:} \\
\texttt{( <0> false [ <0> , <4> ] <5> OR [ <3> ] <4> NOT ( <1> false [ <1> , <2> ] <3> AND <2> true ) ) <|trace|>}

\vspace{2mm}
\hrule
\vspace{2mm}

\textbf{Signpost CoT Trace:} \\
\texttt{<0> F <1> F <2> T <3> F <4> T <5> T answer: T<|endoftext|>}
\end{promptbox}

\subsubsection{S5 Permutation}
\label{subsec:datagen_s5}
\textbf{Data Generation:} For both training and evaluation, initial states are created by randomly assigning five distinct objects to variables A through E. The sequence length is defined by the total number of subsequent pairwise swap operations. Swap operations are randomly sampled for each problem instance. As an additional difficulty, we introduce a "repetitive" variant of the task. In this version, there is a 90\% probability that a step in the prompt simply repeats the same swap operation from the previous step. This creates extended sequences of identical operations followed by abrupt changes to new swaps.

\textbf{CoT Formats.}

\textit{Naive CoT:} Operations are presented in chronological order, with the trace logging the full updated state after each step.

\textit{Signpost CoT:} We sequentially assign signposts to swap operations (using an offset during training) and randomly sample without replacement from the remaining signpost pool for the target variables (A - E). During training, we randomly shuffle the operations in the prompt so the model must rely on tracking signposts rather than positional order. At test time, operations are presented in chronological order for direct comparison with naive CoT.

\textit{Value-change CoT:} Identical to Signpost CoT, except the trace omits the full state. Instead, it logs only modified variables per step (e.g., \texttt{W\_A Dog\_Apple}, where \texttt{A} is the modified variable, \texttt{Dog} is the outgoing object and \texttt{Apple} is the incoming object).
To recover the final assignment of objects, the trace concludes with a resolution step where the model compares the number of incoming and outgoing swaps per variable (treating initialization as an incoming swap). A net positive count identifies the final object in that position.  We implement the object transitions as atomic tokens, as we find that this helps the model to learn this counting logic.

\textbf{Training Setup:}
As the S5 Permutation task includes tokens that disproportionately contribute to task success compared to simple marker tokens (e.g., \texttt{swap} or \texttt{.}), standard SGD can struggle to reliably isolate and learn these sparse, critical operations without falling into local minima. Specifically, we apply a 7$\times$ weight to the cross-entropy loss for comparison tokens (e.g., \texttt{==}, \texttt{<}, \texttt{>}) and signpost tokens, and a 3$\times$ weight to variable identifiers (A through E). Additionally, for the Signpost and Value-change CoT formats, we mix in training samples where the swap operations are not shuffled, to improve generalization to unshuffled sequences at test time.

\begin{promptbox}{Training Formats: S5 Permutation}
\textbf{Naive CoT Prompt:} \\
\texttt{init A Book B Cat C Hat D Dog E Apple operation swap E A . swap A B . swap E C . end . <|trace|>}

\vspace{2mm}
\hrule
\vspace{2mm}

\textbf{Naive CoT Trace:} \\
\texttt{swap E A write A Apple B Cat C Hat D Dog E Book swap A B write A Cat B Apple C Hat D Dog E Book swap E C write A Cat B Apple C Book D Dog E Hat end answer Cat Apple Book Dog Hat<|endoftext|>}

\vspace{2mm}
\hrule
\vspace{2mm}

\textbf{Signpost CoT Prompt:} \\
\texttt{init A Book B Cat C Hat D Dog E Apple start <1> stop <3> operation <1> swap [ <11> E <61> A ] <1> . <2> swap [ <33> A <149> B ] <2> . <3> swap [ <6> E <29> C ] <3> . <|trace|>}

\vspace{2mm}
\hrule
\vspace{2mm}

\textbf{Signpost CoT Trace:} \\
\texttt{load <1> . <1> swap <11> E <61> A write A Apple B Cat C Hat D Dog E Book load <2> . <2> swap <33> A <149> B write A Cat B Apple C Hat D Dog E Book load <3> . <3> swap <6> E <29> C write A Cat B Apple C Book D Dog E Hat load end . end answer Cat Apple Book Dog Hat<|endoftext|>}

\vspace{2mm}
\hrule
\vspace{2mm}

\textbf{Signpost + Value-Change CoT Prompt:} \\
\texttt{init A Book B Cat C Hat D Dog E Apple start <1> stop <3> operation <1> swap [ <154> E <28> A ] <1> . <2> swap [ <73> A <84> B ] <2> . <3> swap [ <31> E <164> C ] <3> . <|trace|>}

\vspace{2mm}
\hrule
\vspace{2mm}

\textbf{Signpost + Value-Change CoT Trace:} \\
\texttt{load <1> . <1> swap <154> E <28> A W\_E Apple\_Book W\_A Book\_Apple load <2> . <2> swap <73> A <84> B W\_A Apple\_Cat W\_B Cat\_Apple load <3> . <3> swap <31> E <164> C W\_E Book\_Hat W\_C Hat\_Book load end res <A> init Book Book IN == OUT Cat IN > OUT Hat IN == OUT Dog IN == OUT Apple IN < OUT final Cat <B> init Cat Book IN == OUT Cat IN < OUT Hat IN == OUT Dog IN == OUT Apple IN > OUT final Apple <C> init Hat Book IN > OUT Cat IN == OUT Hat IN < OUT Dog IN == OUT Apple IN == OUT final Book <D> init Dog Book IN == OUT Cat IN == OUT Hat IN == OUT Dog IN == OUT Apple IN == OUT final Dog <E> init Apple Book IN < OUT Cat IN == OUT Hat IN > OUT Dog IN == OUT Apple IN < OUT final Hat answer Cat Apple Book Dog Hat<|endoftext|>}
\end{promptbox}
\subsubsection{Binary Permutation}
\label{appendix:binary_permutation_formats}
\textbf{Data Generation:}
Data is generated the same way as for S5 Permutation. However, we restrict the object set to two objects that can be assigned to multiple variables. We ensure to generate problem instances where each object occurs at least once, to avoid trivial instances.

\textbf{CoT Formats:}
The naive and signpost CoT are the same as for S5 Permutation. The value-change CoT format an additional keep marker (K\_) to denote that the state for a given variable was not changed (i.e. when the incoming and outgoing object is the same).
\textbf{Training Setup:}
As for the S5 Permutation task, we apply a 7$\times$ weight to the cross-entropy loss for comparison tokens (e.g., \texttt{==}, \texttt{<}, \texttt{>}) and signpost tokens, and a 3$\times$ weight to variable identifiers (A through E). Again, for the Signpost and Value-change CoT formats, we mix in training samples where the swap operations are not shuffled.

\begin{promptbox}{Training Formats: Binary Permutation}
\textbf{Naive CoT Prompt:} \\
\texttt{init A Cat B Dog C Cat D Cat E Dog operation swap E A . swap E D . swap E C . end . <|trace|>}

\vspace{2mm}
\hrule
\vspace{2mm}

\textbf{Naive CoT Trace:} \\
\texttt{swap E A write A Dog B Dog C Cat D Cat E Cat swap E D write A Dog B Dog C Cat D Cat E Cat swap E C write A Dog B Dog C Cat D Cat E Cat end answer Dog Dog Cat Cat Cat<|endoftext|>}

\vspace{2mm}
\hrule
\vspace{2mm}

\textbf{Signpost CoT Prompt:} \\
\texttt{init A Cat B Dog C Cat D Cat E Dog start <1> stop <3> operation <1> swap [ <100> E <20> A ] <1> . <2> swap [ <16> E <152> D ] <2> . <3> swap [ <19> E <51> C ] <3> . <|trace|>}

\vspace{2mm}
\hrule
\vspace{2mm}

\textbf{Signpost CoT Trace:} \\
\texttt{load <1> . <1> swap <100> E <20> A write A Dog B Dog C Cat D Cat E Cat load <2> . <2> swap <16> E <152> D write A Dog B Dog C Cat D Cat E Cat load <3> . <3> swap <19> E <51> C write A Dog B Dog C Cat D Cat E Cat load end . end answer Dog Dog Cat Cat Cat<|endoftext|>}

\vspace{2mm}
\hrule
\vspace{2mm}

\textbf{Signpost + Value-Change CoT Prompt:} \\
\texttt{init A Cat B Dog C Cat D Cat E Dog start <1> stop <3> operation <1> swap [ <100> E <20> A ] <1> . <2> swap [ <16> E <152> D ] <2> . <3> swap [ <19> E <51> C ] <3> . <|trace|>}

\vspace{2mm}
\hrule
\vspace{2mm}

\textbf{Signpost + Value-Change CoT Trace:} \\
\texttt{load <1> . <1> swap <100> E <20> A W\_E Dog\_Cat W\_A Cat\_Dog load <2> . <2> swap <16> E <152> D K\_E K\_D load <3> . <3> swap <19> E <51> C K\_E K\_C load end res <A> init Cat IN > OUT final Dog <B> init Dog IN == OUT final Dog <C> init Cat IN == OUT final Cat <D> init Cat IN == OUT final Cat <E> init Dog IN < OUT final Cat answer Dog Dog Cat Cat Cat<|endoftext|>}
\end{promptbox}

\begin{figure}[h!]
    \centering
    \includegraphics[width=1.0\linewidth]{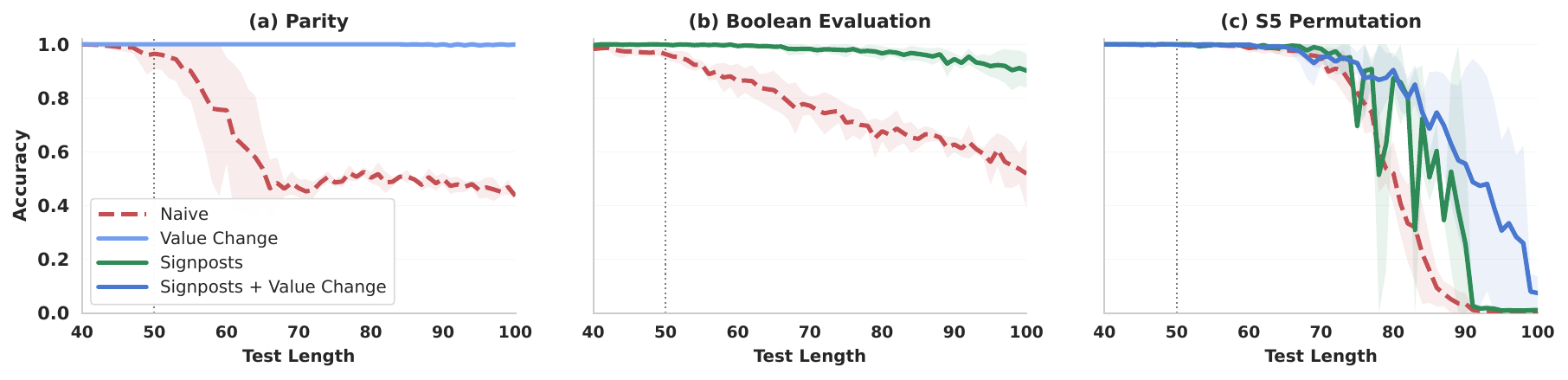}
    \caption{Evaluation curves for training length 50 for all our tasks. We observe similar trends as seen in Figure~\ref{fig:50_results_main} for Parity and Boolean Evaluation. However for S5, we hypothesize that our models are undertrained and would require more training for seeing generalization. We do see generalization in one seed for S5 (with value change and signposts), but overall we found S5 hard to train on lengths 50. We document and comment on this in Section~\ref{appendix:s5_discussion}.}
    \label{fig:len_50_curve}
\end{figure}

\subsection{Binary Permutation Experiments}
\label{appendix:binary_permutation_experiments}
To highlight the potential of learning value-change formats, we evaluate \textsc{Binary Permutation} as an ablation of the \textsc{S5 Permutation} task. This task restricts the object set to just two objects which can be assigned to more than one variable (task description in Appendix \ref{appendix:binary_permutation_description}). As a consequence, this task produces more repetitions of objects in the state updates, as well as swap operations that do not alter the state (i.e. when the incoming and outgoing object is the same). In Figure \ref{fig:permutation_binary_len30} we report the results for training length 30 on this task. The results confirm that combining signpost tokens with value-change encoding yields the strongest length generalization, outperforming both the naive CoT and signpost-only formats, across both the standard and repetitive variant (described in Section~\ref{subsec:datagen_s5}) of the task.
\begin{figure}[h]
    \centering
    \includegraphics[width=1.0\linewidth]{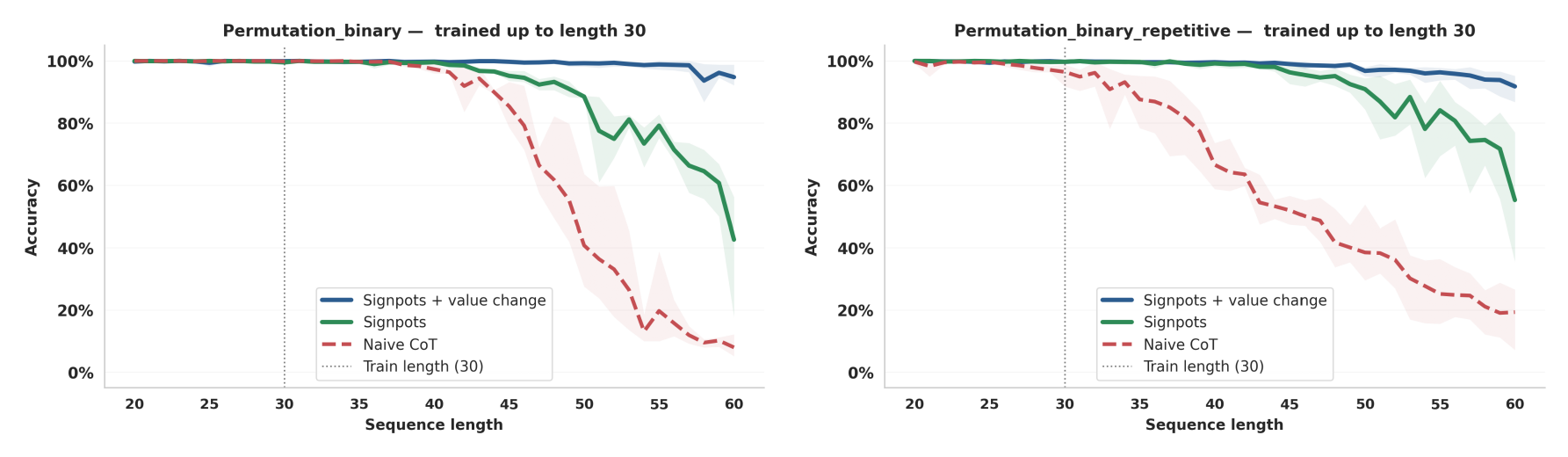}
    \caption{Evaluation curves for training length 30 on the Binary Permutation task. Compared to S5 Permutation, this task introduces the added difficulty of non-unique object retrieval when updating states, caused by frequent object repetitions in the trace. Value-Change CoT successfully mitigates this non-unique retrieval problem.}
    \label{fig:permutation_binary_len30}
\end{figure}
\begin{figure}[h]
    \centering
    \includegraphics[width=1.0\linewidth]{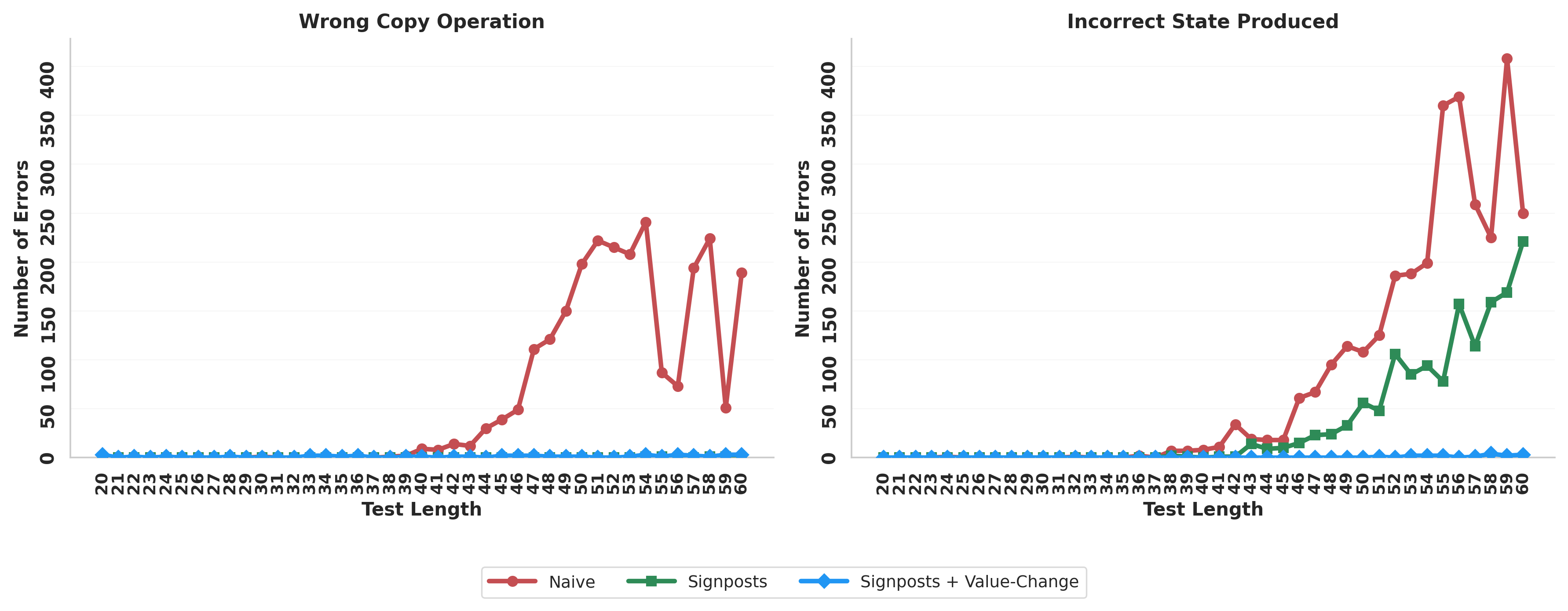}
    \caption{Frequency of first-occurring errors across test lengths for our three evaluated CoT formats on Binary Permutation. Wrong copy operations (incorrect signpost retrieval or variable pairs) and incorrect state updates are the most prevalent error types. At longer test lengths, incorrect state updates become the dominant failure mode for Naive and Signpost CoT, whereas the Value-Change format successfully mitigates both error types.}
    \label{fig:permutation_binary_error_analysis}
\end{figure}

We further quantitatively confirm that the degradation of performance of the signpost-only format is driven by highly repetitive state updates through an error analysis (see Figure \ref{fig:permutation_binary_error_analysis}). To this end, we check the best performing seed for each CoT scheme for the first occurring error in the trace across all test lengths. We find that the majority of errors are caused by incorrect copy operations (i.e. copying the wrong signpost token or the wrong variables in a swap operation) and incorrect state updates (i.e. the model outputs the correct swap operation, but fails to correctly update the state from the previous state). For the value-change CoT format, an incorrect state update is registered when the model outputs a keep-token instead of a write-token (and vice-versa). Generally, we observe that signpost tokens substantially reduce wrong copy errors relative to the naive CoT, consistent with their theoretical role in addressing the copying barrier. For longer test lengths, incorrect state updates become the most dominant error case for both the signpost and naive CoT format. In contrast, the number of such errors remain low for the value-change format, demonstrating that such an encoding successfully helps to overcome the non-unique retrieval barrier.

\subsection{Experience with Training S5}
\label{appendix:s5_discussion}

Empirically, we found \textsc{S5 Permutation} to be harder to train than \textsc{Parity} and \textsc{Boolean Evaluation}: it was more seed-sensitive, and required more optimization effort, and as Figure \ref{fig:len_50_curve} shows, exhibits weaker length generalization than others on length 50. 
In particular, we found that models struggle at retrieving the correct signpost or object identifiers (A - E) on longer test sequences.
A plausible explanation for this is that \textsc{S5 Permutation} requires significantly more tokens per reasoning step, leading to noisier optimization. Furthermore, this leads to extensive training times and our models may be undertrained for this length. As shown in Figure \ref{fig:ood_id_overview}, we do observe improved accuracy in checkpoints as training progresses.
We therefore believe that more training data should help in achieving more stable and even better generalization.

\subsection{Prompting Experiments}\label{appendix:prompting-experiments}

For the prompting experiments we aim to cover a range of LLM architectures and sizes. To this end, we evaluate the three models Qwen2.5-7B-Instruct-Turbo, Mistral-Small-24B-Instruct-2501 and Llama-3.3-70B-Instruct-Turbo via the TogetherAI API \footnote{https://www.together.ai/}. 

\textbf{Signpost Experiment:} We slightly adapt the task formats for our algorithmic tasks to make them more natural looking. While a more minimalist task format is advantageous for the from-scratch experiments (as it reduces the total number of processed tokens), pretrained LLMs benefit from formats that align more closely with the natural language distribution of their pretraining data.
For instance, we remove unnecessary whitespaces (which we needed for the whitespace tokenizer in the from-scratch experiments), replace the signpost tokens from the from-scratch experiments with natural numbers, and add "=" signs for the S5 Permutation task, to make models better understand the assignment of objects to variables.

Consistent with our from-scratch evaluations, all models are evaluated using greedy decoding to ensure reproducible results. For both the baseline and signpost CoT formats, we evaluate the models in a 5-shot setting, providing five in-context examples that mirror the respective CoT traces used in our from-scratch training setups.

\begin{promptbox}{Prompt Format: S5 Permutation (Naive vs. \textcolor{blue!80!black}{Signpost Tokens})}
You are a state-tracking agent evaluating the contents of boxes labeled A through E. You are given an initial state ('init') assigning an object to each box, followed by a \textcolor{blue!80!black}{numbered} sequence of operations to swap the contents of two boxes. You must track the state step-by-step \textcolor{blue!80!black}{using the provided operation indices}. 

For each operation, you must repeat the \textcolor{blue!80!black}{numbered} swap command, and then output 'write' followed by the updated state of all boxes. \textcolor{blue!80!black}{Use the index hints to find the correct swap operation in the trace.} 

When the operation sequence is fully resolved, output the final state from the last operation. Finish your answer with 'finished'. Do not output any natural language. 

Below are examples of how to correctly execute the operations.
\end{promptbox}

\begin{promptbox}{Prompt Format: Boolean Evaluation (Naive CoT)}
Your task is to compute the truth value of a boolean formula. You will be given a boolean formula containing the operators NOT, AND, OR, and the literals true and false. Evaluate the formula step by step in post-order. Evaluate the leaves first and the root last. For each step, write down the currently evaluated subformula and then calculate its truth value.

Your final answer must be exactly "answer T" or "answer F" with no other formatting. Give your final answer in this format.

When the formula is fully evaluated, output the final state from the last operation. Finish your answer with 'finished'.

Below are examples of how to correctly execute the operations.
\end{promptbox}

\begin{promptbox}{Prompt Format: Boolean Evaluation (Signpost Tokens)}
Your task is to compute the truth value of a boolean formula. You will be given a boolean formula containing the operators NOT, AND, OR, and the literals true and false.

The formula has been annotated with index hints that indicate the order in which the formula must be evaluated.

\textbf{Syntax Rules:}
\begin{itemize}
    \item Literals: ``$\langle$index$\rangle$ $\langle$boolean$\rangle$'' (e.g., ``0 true'')
    \item Unary Operations: ``[ $\langle$operand\_index$\rangle$ ] $\langle$index$\rangle$ NOT'' (e.g., ``[ 0 ] 1 NOT'')
    \item Binary Operations: ``[ $\langle$left\_index$\rangle$ , $\langle$right\_index$\rangle$ ] $\langle$index$\rangle$ $\langle$AND/OR$\rangle$'' (e.g., ``[ 0 , 1 ] 2 AND'')
\end{itemize}

For each step, output the index number followed by 'T' if the respective subformula evaluates to true, or 'F' if it evaluates to false. 

When the formula is fully evaluated, output the final state from the last operation. Finish your answer with 'finished'. Do not output any natural language.

Below are examples of how to correctly execute the operations.
\end{promptbox}

\textbf{Value-Change:} 
To evaluate the effectiveness of logging value changes, we designed a zero-shot Variable Assignment Evaluation task. Crucially, unlike the previous signpost experiments where models were asked to generate step-by-step intermediate reasoning (CoT), here we explicitly force the models to act as internal state trackers and output \emph{only} the final answer directly. 

We evaluate the models on two distinct input formats to see if augmenting the input context alone improves internal state evaluation: a standard code format with line numbers, and a \emph{value-change} format where the input program is augmented with line numbers and inline comments explicitly tracing the state transitions of the variables (e.g., \texttt{\# node\_b: 1 -> 0}).

\begin{promptbox}{Prompt Format: Variable Assignment (Standard)}
\textbf{System Prompt:} You are a very careful and precise assistant. You always follow the instructions and solve tasks yourself. You never generate code. You also give the answer directly whenever possible without trying to generate any intermediate Chain of Thought.

\textbf{User:} \\
Program: \\
1. node\_a = 9 \\
2. node\_b = 4 \\
3. node\_c = 2 \\
4. node\_m = 3 \\
5. node\_q = 2 \\
6. node\_a = 0 \\
7. node\_p = 0 \\
8. node\_d = 0 \\
9. node\_e = 5 \\
10. node\_c = 8 \\
11. node\_e = 4 \\
12. node\_m = 7 \\
13. node\_q = 1 \\
Query:
print(node\_a, node\_m, node\_b) \\ 
Return only the final answer in $\langle$ output $\rangle$ $\langle$output $\rangle$ tags.
\end{promptbox}

\begin{promptbox}{Prompt Format: Variable Assignment with Value Change }
\textbf{System Prompt:} You are a very careful and precise assistant. You always follow the instructions and solve tasks yourself. You never generate code. You also give the answer directly whenever possible without trying to generate any intermediate Chain of Thought.

\textbf{User:} \\
Program: \\
1. node\_a = 9 \\
2. node\_b = 4 \\
3. node\_c = 2 \\
4. node\_m = 3 \\
5. node\_q = 2 \\
6. node\_a = 0 \quad \# node\_a: 9 $\to$ 0  \\
7. node\_p = 0 \\
8. node\_d = 0 \\
9. node\_e = 5 \\
10. node\_c = 8 \quad \# node\_c: 2 $\to$ 8 \\
11. node\_e = 4 \quad \# node\_e: 5 $\to$ 4 \\
12. node\_m = 7 \quad \# node\_m: 3 $\to$ 7 \\
13. node\_q = 1 \quad \# node\_q: 2 $\to$ 1 \\
Query:
print(node\_a, node\_m, node\_b) \\ 
Return only the final answer in $\langle$ output $\rangle$ $\langle$output $\rangle$ tags.
\end{promptbox}

\subsection{Hyperparameter Settings \& Model Selection}\label{appendix:hyperparams}
This section reports the hyperparameter configurations, the top three seeds based on the validation set, as well as graphs for the ID validation and OOD validation performance throughout training.

\begin{table}[t]
\centering
\small
\setlength{\tabcolsep}{3pt}

\begin{tabular}{llccccc}
\toprule
\textbf{Task (Train Length)}
& \textbf{CoT Format}
& \textbf{Rank}
& \textbf{Seed}
& \begin{tabular}[c]{@{}c@{}}
    \textbf{Learning}\\
    \textbf{Rate}
  \end{tabular}
& \begin{tabular}[c]{@{}c@{}}
    \textbf{Weight}\\
    \textbf{Decay}
  \end{tabular}
& \begin{tabular}[c]{@{}c@{}}
    \textbf{Repetitive}\\
    \textbf{Ratio}
  \end{tabular}
\\
\midrule

\multirow{9}{*}{\textbf{S5 Permutation (Length 30)}}
& \multirow{3}{*}{Naive}
& 1 & [46] & [5e-4] & [0.1]  & [0.05] \\
& & 2 & [47] & [3e-4] & [0.01] & [0.05] \\
& & 3 & [50] & [5e-4] & [0.1]  & [0.05] \\
\cmidrule{2-7}

& \multirow{3}{*}{Signpost}
& 1 & [50] & [3e-4] & [0.01] & [0.05] \\
& & 2 & [44] & [5e-4] & [0.1]  & [0.05] \\
& & 3 & [46] & [5e-4] & [0.1]  & [0.05] \\
\cmidrule{2-7}

& \multirow{3}{*}{Signpost + Value-Change}
& 1 & [45] & [5e-4] & [0.1] & [0.05] \\
& & 2 & [49] & [5e-4] & [0.1] & [0.05] \\
& & 3 & [47] & [5e-4] & [0.1] & [0.15] \\
\midrule

\multirow{9}{*}{\textbf{S5 Permutation (Length 50)}}
& \multirow{3}{*}{Naive}
& 1 & [48] & [5e-4] & [0.01] & [0.05] \\
& & 2 & [45] & [5e-4] & [0.01] & [0.05] \\
& & 3 & [44] & [3e-4] & [0.01] & [0.05] \\
\cmidrule{2-7}

& \multirow{3}{*}{Signpost}
& 1 & [45] & [3e-4] & [0.01] & [0.05] \\
& & 2 & [44] & [3e-4] & [0.01] & [0.05] \\
& & 3 & [48] & [3e-4] & [0.01] & [0.05] \\
\cmidrule{2-7}

& \multirow{3}{*}{Signpost + Value-Change}
& 1 & [44] & [5e-4] & [0.01] & [0.05] \\
& & 2 & [46] & [5e-4] & [0.01] & [0.05] \\
& & 3 & [49] & [5e-4] & [0.01] & [0.05] \\
\midrule

\multirow{9}{*}{\textbf{Binary Permutation (Length 30)}}
& \multirow{3}{*}{Naive}
& 1 & [50] & [5e-4] & [0.1] & [0.05] \\
& & 2 & [48] & [5e-4] & [0.1] & [0.15] \\
& & 3 & [49] & [5e-4] & [0.1] & [0.05] \\
\cmidrule{2-7}

& \multirow{3}{*}{Signpost}
& 1 & [48] & [5e-4] & [0.1] & [0.15] \\
& & 2 & [50] & [5e-4] & [0.1] & [0.15] \\
& & 3 & [46] & [5e-4] & [0.1] & [0.15] \\
\cmidrule{2-7}

& \multirow{3}{*}{Signpost + Value-Change}
& 1 & [45] & [5e-4] & [0.1] & [0.15] \\
& & 2 & [49] & [5e-4] & [0.1] & [0.05] \\
& & 3 & [46] & [5e-4] & [0.1] & [0.05] \\

\bottomrule
\end{tabular}

\caption{
Top three performing hyperparameter configurations and their corresponding
random seeds for the permutation tasks. Configurations are ranked according
to the performance of their best validation-set checkpoint. }
\label{tab:top_configs_permutation}
\end{table}

\begin{table}[t]
\centering
\small
\setlength{\tabcolsep}{4pt}

\begin{tabular}{llcccc}
\toprule
\textbf{Task (Train Length)}
& \textbf{CoT Format}
& \textbf{Rank}
& \textbf{Seed}
& \textbf{Learning Rate}
& \textbf{Weight Decay} \\
\midrule

\multirow{6}{*}{\textbf{Parity (Length 30)}}
& \multirow{3}{*}{Naive}
& 1 & [48] & [5e-4] & [0.1] \\
& & 2 & [50] & [5e-4] & [0.1] \\
& & 3 & [49] & [5e-4] & [0.1] \\
\cmidrule{2-6}
& \multirow{3}{*}{Value-Change}
& 1 & [44] & [3e-4] & [0.1] \\
& & 2 & [45] & [1e-4] & [0.01] \\
& & 3 & [46] & [1e-4] & [0.01] \\
\midrule

\multirow{6}{*}{\textbf{Parity (Length 50)}}
& \multirow{3}{*}{Naive}
& 1 & [47] & [5e-4] & [0.1] \\
& & 2 & [45] & [3e-4] & [0.1] \\
& & 3 & [44] & [3e-4] & [0.1] \\
\cmidrule{2-6}
& \multirow{3}{*}{Value-Change}
& 1 & [44] & [3e-4] & [0.1] \\
& & 2 & [45] & [1e-4] & [0.01] \\
& & 3 & [46] & [3e-4] & [0.1] \\
\midrule

\multirow{6}{*}{\textbf{Boolean (Length 30)}}
& \multirow{3}{*}{Naive}
& 1 & [47] & [3e-4] & [0.1] \\
& & 2 & [46] & [5e-4] & [0.1] \\
& & 3 & [45] & [3e-4] & [0.1] \\
\cmidrule{2-6}
& \multirow{3}{*}{Signpost}
& 1 & [46] & [5e-4] & [0.1] \\
& & 2 & [49] & [3e-4] & [0.1] \\
& & 3 & [50] & [5e-4] & [0.1] \\
\midrule

\multirow{6}{*}{\textbf{Boolean (Length 50)}}
& \multirow{3}{*}{Naive}
& 1 & [46] & [3e-4] & [0.1] \\
& & 2 & [48] & [3e-4] & [0.1] \\
& & 3 & [49] & [5e-4] & [0.1] \\
\cmidrule{2-6}
& \multirow{3}{*}{Signpost}
& 1 & [49] & [5e-4] & [0.1] \\
& & 2 & [47] & [5e-4] & [0.1] \\
& & 3 & [44] & [5e-4] & [0.1] \\

\bottomrule
\end{tabular}

\caption{
Top three performing hyperparameter configurations and their corresponding
random seeds for the Parity and Boolean Evaluation tasks. Configurations
are ranked according to the performance of their best validation-set
checkpoint.
}
\label{tab:top_configs_parity_boolean}
\end{table}

\begin{figure}[ht]
    \centering
    \includegraphics[width=0.8\linewidth]{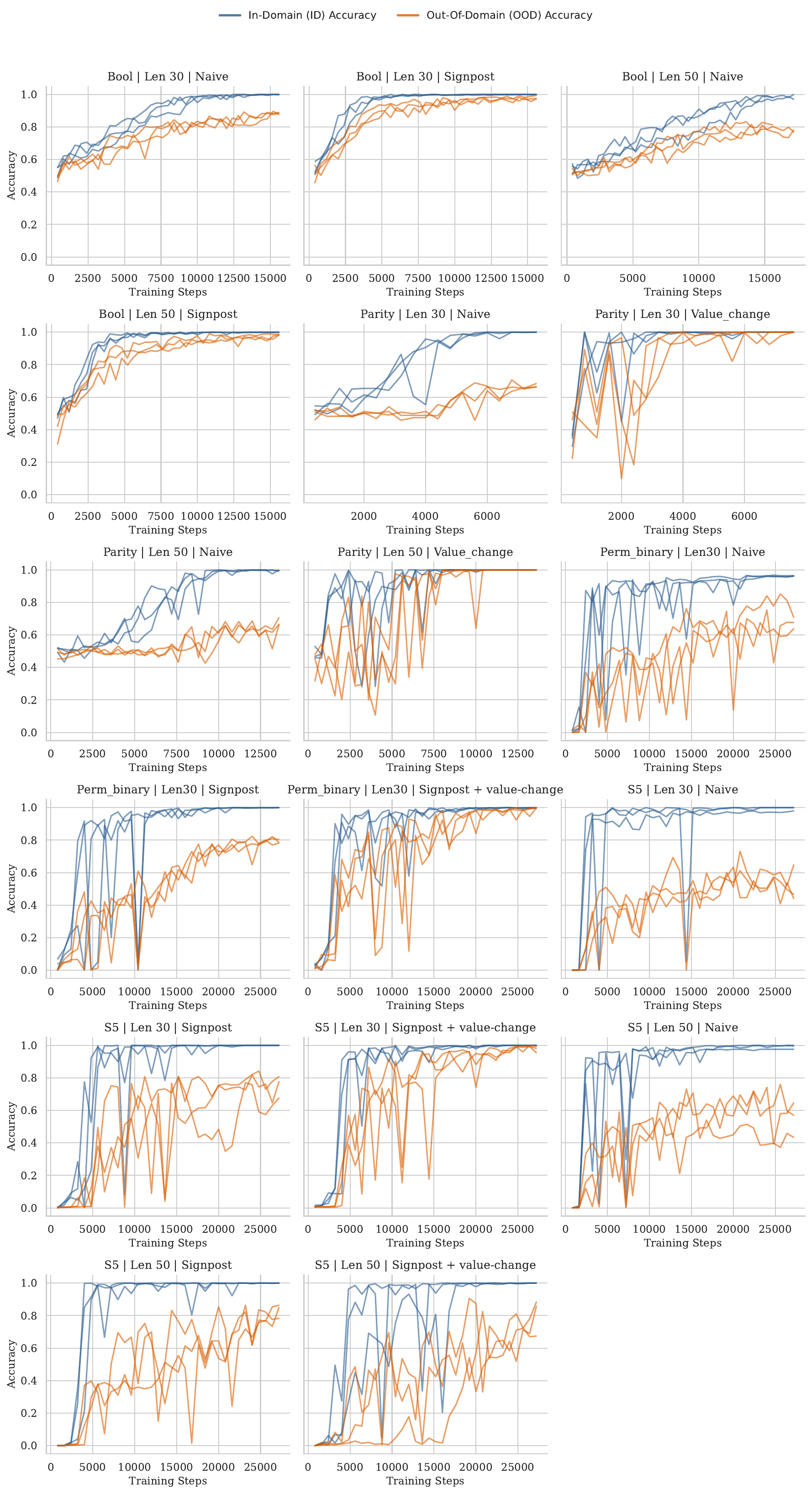}
    \caption{Validation dynamics for the top 3 performing seeds (based on validation accuracy) across all evaluated tasks, training lengths, and CoT formats. We plot both ID validation accuracy (blue) and OOD validation accuracy (orange) over the course of training. }
    \label{fig:ood_id_overview}
\end{figure}

\begin{table*}[h]
\centering
\small
\begin{tabular}{
    @{}
    >{\raggedright\arraybackslash}p{0.35\textwidth}
    >{\raggedright\arraybackslash}p{0.55\textwidth}
    @{}
}
\toprule
\multicolumn{2}{c}{\textbf{Common Parameters}} \\
\midrule
Model Architecture & GPT-2 (6 layers, 8 attention heads, 512 embedding dimension) \\
Number of Positional Encodings & 2048 (4096 for S5 length 50) \\
Parameter Count & $\sim$25M \\
Positional Encodings & Absolute Positional Encodings (APE) \\
Tokenizer & Task-specific whitespace tokenizer \\
Decoding Method & Greedy \\
Learning Rate Scheduling & Cosine decay with 0.1 min LR \\
Learning Rate Warmup Ratio & 0.05 \\
Learning Rate Warmup Scheduling & Linear \\
Dropout & 0.0 \\
Optimizer & AdamW \\
Batch Size & 256 \\
Seeds & 44, 45, 46, 47, 48, 49, 50 \\
\midrule
\multicolumn{2}{c}{\textbf{Task-Specific Configurations}} \\
\midrule

\textbf{Parity} & \\
LR \& Weight Decay Pairs &
[(1e-4, 0.01), (3e-4, 0.1), (5e-4, 0.1)] \\
Training Samples &
[2M (for train length 30), 3.5M (for train length 50)] \\

\midrule
\textbf{Boolean Evaluation} & \\
LR \& Weight Decay Pairs &
[(1e-4, 0.01), (3e-4, 0.1), (5e-4, 0.1)] \\
Training Samples & [4.5M] \\
Max Trained Signposts &
[65 (for train length 30), 105 (for train length 50)] \\

\midrule
\textbf{S5 Permutation (Len 30)} & \\
Evaluated Configurations (LR, WD, Rep. Ratio) &
[(3e-4, 0.01, 0.05), (5e-4, 0.1, 0.05),
(5e-4, 0.1, 0.15)] \\
Training Samples & [7M] \\
Max Trained Signposts & [185] \\
Ordered Data Percentage & [0.05] \\

\midrule
\textbf{S5 Permutation (Len 50)} & \\
Evaluated Configurations (LR, WD, Rep. Ratio) &
[(3e-4, 0.01, 0.05), (5e-4, 0.01, 0.05),
(5e-4, 0.1, 0.05)] \\
Training Samples & [7M] \\
Max Trained Signposts & [350] \\
Ordered Data Percentage & [0.10] \\

\midrule
\textbf{Binary Permutation (Len 30)} & \\
Evaluated Configurations (LR, WD, Rep. Ratio) &
[(3e-4, 0.01, 0.05), (5e-4, 0.1, 0.05),
(5e-4, 0.1, 0.15)] \\
Training Samples & [7M] \\
Max Trained Signposts & [185] \\
Ordered Data Percentage & [0.05] \\

\bottomrule
\end{tabular}
\caption{Experimental setup and hyperparameter search space across the four evaluated tasks. Common architectural parameters are shared across all models. For each task-specific configuration, hyperparameter ranges evaluated during the search are denoted in square brackets.}
\label{tab:hyperparams}
\end{table*}

\end{document}

%% file: math_commands.tex
\usepackage{amsmath,amsfonts,bm}

\def\eqref#1{equation~\ref{#1}}
\def\1{\bm{1}}

\def\vc{{\bm{c}}}

\DeclareMathAlphabet{\mathsfit}{\encodingdefault}{\sfdefault}{m}{sl}
\SetMathAlphabet{\mathsfit}{bold}{\encodingdefault}{\sfdefault}{bx}{n}

\newcommand{\Omit}[1]{}

\newcommand{\defeq}{\leftarrow}